\documentclass[11pt,titlepage]{article}
\usepackage{latexsym}
\usepackage{amsfonts}
\usepackage{amsmath}
\usepackage{mathtools}
\usepackage{amssymb}
\usepackage{comment}

\usepackage{amssymb}
\usepackage[margin=1in]{geometry}
\usepackage{graphicx}
\usepackage{tikz}
\usepackage{framed}
\usepackage{collectbox}
\usetikzlibrary{arrows, decorations.markings}
\usetikzlibrary{fit}					
\usetikzlibrary{backgrounds}

\usepackage{booktabs}       
\usepackage{nicefrac}       
\usepackage{microtype}      
\usepackage{subcaption}
\usepackage{color}
\usepackage{floatrow}

\newenvironment{proof}{\paragraph{Proof:}}{\hfill$\square$}
\newcommand{\todo}[2]{\textcolor{red}{TODO(#1): #2}}

\def\l{,\ldots,}
\newtheorem{theorem}{}[section]

\tikzstyle{every node}=[circle, draw, fill=black!50, inner sep=0pt, minimum width=20pt]

\tikzstyle{input}=[circle,
                                    thick,
                                    minimum size=2.2cm,
                                    draw=black!80,
                                    fill=white!20]

\tikzstyle{input2}=[circle,
                                    thick,
                                    minimum size=1.0cm,
                                    draw=black!80,
                                    fill=white!20]

\tikzstyle{matrx}=[rectangle,
                                    thick,
                                    minimum size=1.8cm,
                                    draw=black!80,
                                    fill=white!20]

\tikzstyle{matrx2}=[rectangle,
                                    thick,
                                    minimum size=1.5cm,
                                    draw=black!80,
                                    fill=gray!20]

\tikzstyle{vecArrow} = [thick, decoration={markings,mark=at position
   1 with {\arrow[semithick]{open triangle 60}}},
   double distance=1.8pt, shorten >= 5.5pt,
   preaction = {decorate},
   postaction = {draw,line width=1.4pt, white,shorten >= 4.5pt}]
\tikzstyle{innerWhite} = [semithick, white,line width=1.4pt, shorten >= 4.5pt]

\tikzstyle{background}=[rectangle,
                                                fill=gray!10,
                                                inner sep=0.2cm,
                                                rounded corners=5mm]

\newtheorem{lemma}{Lemma}
\newtheorem{definition}{Definition}
\newtheorem{remark}{Remark}

\title{TripleSpin - a generic compact paradigm for fast
machine learning computations}
\author{
Krzysztof Choromanski\\
Google Research\\
\texttt{kchoro@google.com}
\and
Francois Fagan\\
Columbia University\\
\texttt{ff2316@columbia.edu}
\and
Cedric Gouy-Pailler\\
CEA, LIST, LADIS \\
\texttt{cedric.gouy-pailler@cea.fr}
\and
Anne Morvan \\
CEA, LIST, LADIS \& Universite Paris-Dauphine \\
\texttt{anne.morvan@cea.fr} \\
\and
Tamas Sarlos \\
Google Research \\
\texttt{stamas@google.com} \\
\and
Jamal Atif \\
Universite Paris-Dauphine \\
\texttt{jamal.atif@dauphine.fr} 
}

\date{May 15, 2016}
\begin{document}

\maketitle

\begin{abstract}
We present a generic compact computational framework relying on structured random matrices that can be applied to
speed up several machine learning algorithms with almost no loss of accuracy. The applications include 
new fast LSH-based algorithms, efficient kernel computations via random feature maps, convex optimization algorithms, quantization techniques and many more.
Certain models of the presented paradigm are even more compressible since they apply only bit matrices.
This makes them suitable for deploying on mobile devices.    
All our findings come with strong theoretical guarantees. In particular, as a byproduct of the presented techniques and by using relatively new Berry-Esseen-type CLT for random vectors, we 
give the first theoretical guarantees for one of the most efficient existing LSH algorithms based on the $\textbf{HD}_{3}\textbf{HD}_{2}\textbf{HD}_{1}$
structured matrix \cite{indyk2015}. These guarantees as well as theoretical results for other
aforementioned applications follow from the same general theoretical principle that we present in the paper.
Our structured family contains as special cases all previously considered structured schemes, including the recently introduced $P$-model \cite{chor_sind_2016}.
Experimental evaluation confirms the accuracy and efficiency of TripleSpin matrices.
\end{abstract}

\section{Introduction}
Consider a randomized machine learning algorithm $\mathcal{A}_{\textbf{G}}(\mathcal{X})$ on a dataset $\mathcal{X}$ and parameterized by the Gaussian matrix $\textbf{G}$ with i.i.d. entries taken from $\mathcal{N}(0,1)$.
We assume that $\textbf{G}$ is used to calculate Gaussian projections that are then
passed to other, possibly highly nonlinear functions.


Many machine learning algorithms are of this form. Examples include several variants of the Johnson-Lindenstrauss Transform applying random projections to reduce data dimensionality while approximately preserving Euclidean distance \cite{liberty_jlt, ailon2013}, quantization techniques using random projection trees, where splitting in each node is determined by a projection of data onto Gaussian direction \cite{dasgupta08}, algorithms solving convex optimization problems with random sketches of Hessian matrices \cite{pilanci, pilanci2}, kernel approximation techniques applying random feature maps produced from linear projections with Gaussian matrices followed by nonlinear mappings \cite{chor_sind_2016, rahimi, choromanska2015binary}, several LSH-schemes  \cite{indyk2015,charikar, terasawa} (such as some of the most effective cross-polytope LSH methods) and many more.

If data is high-dimensional then computing random projections $\{\textbf{G}\textbf{x} : \textbf{x} \in \mathcal{X} \}$ for $\textbf{G} \in \mathbb{R}^{m \times n}$ in time $\Theta(mn|\mathcal{X}|)$ often occupies a significant fraction of the overall 
run time. Furthermore, storing matrix $\textbf{G}$ frequently becomes a bottleneck in terms of space complexity. In this paper 
we propose a ``structured variant'' of the algorithm $\mathcal{A}$, where Gaussian matrix $\textbf{G}$ is replaced by a structured matrix 
$\textbf{G}_{struct}$ taken from the defined by us \textit{TripleSpin}-family of matrices. The name comes from the fact that 
each such matrix is a product of three other matrices, building components, which include rotations.

Replacing $\textbf{G}$ by $\textbf{G}_{struct}$ gives computational speedups since $\textbf{G}_{struct}\textbf{x}$ can be calculated in $o(mn)$ time: with the use of techniques such as Fast Fourier Transform, time complexity is reduced to $O(n\log m)$.
Furthermore, with matrices from the \textit{TripleSpin}-family space complexity can be also substantially reduced --- to sub-quadratic, usually at most linear, sometimes even constant. 

To the best of our knowledge, we are the first to provide a comprehensive theoretical explanation of the effectiveness of the structured approach. So far such an
explanation was given only for some specific applications and specific structured matrices. The proposed \textit{TripleSpin}-family contains all previously considered structured matrices as special cases, including the recently introduced $P$-model \cite{chor_sind_2016}, yet its flexible three-block structure provides mechanisms for constructing many others.
Our structured model is also the first one that proposes purely discrete Hadamard-based constructions with strong theoretical guarantees. We empirically show that these surpass their unstructured counterparts.

As a byproduct, we provide a theoretical explanation of the efficiency of the cross-polytope LSH method based on the $\textbf{HD}_{3}\textbf{HD}_{2}\textbf{HD}_{1}$ structured matrix \cite{indyk2015}, where $\textbf{H}$ stands for the Hadamard matrix and $\textbf{D}_{i}$s are random diagonal $\pm 1$-matrices. Thus we solve the open problem posted in \cite{indyk2015}. These guarantees as well as theoretical results for other aforementioned applications arise from the same general theoretical principle that we present in the paper. Our theoretical methods apply relatively new Berry-Esseen type Central Limit Theorem results
for random vectors.

\section{Related work}
\label{sec:related_work}

Structured matrices were previously used mainly in the context of the Johnson-Lindenstrauss Transform (JLT), where the goal is to linearly transform 
high-dimensional data and embed it into a much lower dimensional space in such a way that the Euclidean distance is approximately preserved \cite{liberty_jlt, ailon2013, ailon2006approximate}. 
Most of these structured constructions involve sparse or circulant matrices \cite{ailon2006approximate,  
vybiral2011variant} providing computational speedups and space compression.



Specific structured matrices were used to approximate angular distance and Gaussian kernels \cite{choromanska2015binary, feng}. 
Very recently \cite{chor_sind_2016} structured matrices coming from the so-called \textit{P-model}, were applied to speed up random feature map computations of some special kernels (angular, arc-cosine and Gaussian). The presented techniques did not work for discrete structured constructions such as the $\textbf{HD}_{3}\textbf{HD}_{2}\textbf{HD}_{1}$ structured matrix since they focus on matrices with low (polylog) chromatic number of the corresponding coherence graphs and these do not include matrices such as $\textbf{HD}_{3}\textbf{HD}_{2}\textbf{HD}_{1}$ or their
direct non-discrete modifications.

The $TripleSpin$-mechanism gives in particular a highly parameterized family of structured methods for approximating general kernels with random feature maps.
Among them are purely discrete computational schemes providing the most aggressive
compression with just  minimal loss of accuracy.

Several LSH methods use random Gaussian matrices to construct compact codes of data points and in turn speed up such tasks as approximate nearest neighbor search. Among them are
cross-polytope methods proposed in \cite{terasawa}. In the angular cross-polytope setup the hash family $\mathcal{H}$ is designed for points taken from the unit sphere $S^{n-1}$, where $n$ stands for data dimensionality. To construct hashes a random matrix $\textbf{G} \in \mathbb{R}^{n \times n}$ with i.i.d. Gaussian entries 
is built. The hash of a given data point $\textbf{x} \in S^{n-1}$ is defined as:
$
h(\textbf{x}) = \eta(\frac{\textbf{G}\textbf{x}}{\|\textbf{G}\textbf{x}\|_{2}}),
$
where $\eta(\textbf{y})$ returns the closest vector to $\textbf{y}$ from the set $\{\pm 1 \textbf{e}_{i}\}_{1 \leq i \leq n}$, where $\{\textbf{e}_{i}\}_{1 \leq i \leq n}$ stands for the canonical basis.
The fastest known variant of the cross-polytope LSH \cite{indyk2015} replaces unstructured Gaussian matrix $\textbf{G}$
with the product $\textbf{HD}_{3}\textbf{HD}_{2}\textbf{HD}_{1}$. No theoretical guarantees regarding that variant were known. We provide a theoretical explanation here. As in the previous setting, matrices from the $TripleSpin$ model lead to several fast structured cross-polytope LSH algorithms not considered before. 



Recently a new method for speeding up algorithms solving convex optimization problems by approximating Hessian matrices (a bottleneck of the overall computational pipeline) with their random sketches was proposed \cite{pilanci, pilanci2}. 
One of the presented sketches, the so-called \textit{Newton Sketch}, leads to the sequence of iterates $\{\textbf{x}^{t}\}_{t=0}^{\infty}$ for optimizing given function $f$ given by the following recursion:
\begin{equation}
\textbf{x}^{t+1} = argmin_{\textbf{x}} 
\{\frac{1}{2}\|\textbf{S}^{t}\nabla^{2}f(\textbf{x}^{t})^{\frac{1}{2}}(\textbf{x}-\textbf{x}^{t})\|_{2}^{2} + (\nabla f(\textbf{x}^{t}))^{T} \cdot (\textbf{x}-\textbf{x}^{t})\},
\end{equation}

where $\{\textbf{S}^{t} : t=1,...,\}$ is the set of the so-called \textit{sketch matrices}. Initially the sub-Gaussian sketches based on i.i.d. sub-Gaussian random variables
were used. The disadvantage of the sub-Gaussian sketches $\textbf{S} \in \mathbb{R}^{m \times n}$ lies in the fact that computing the sketch of the given matrix 
$\textbf{A} \in \mathbb{R}^{n \times d}$, needed for convex optimization with sketches, requires $O(mnd)$ time. Thus the method is too slow in practice.
This is the place, where structured matrices can be applied. Some structured approaches were already considered in \cite{pilanci}, where sketches based on randomized orthonormal 
systems were proposed. In that approach a sketching matrix $\textbf{S} \in \mathbb{R}^{m \times n}$ is constructed by sampling i.i.d. rows of the form 
$\textbf{s}^{T} = \sqrt{n} \textbf{e}^{T}_{j} \textbf{HD}$ with probability $\frac{1}{n}$ for $j=1,...,n$, where $\textbf{e}^{T}_{j}$ is chosen uniformly at random
from the set of the canonical vectors and $\textbf{D}$ is a random diagonal $\pm 1$-matrix.
We show that our class of $\textit{TripleSpin}$ matrices can also be used in that setting.

\section{The \textit{TripleSpin}-family}
\label{sec:model}

We now present the \textit{TripleSpin}-family. 
If not specified otherwise, the random diagonal matrix $\textbf{D}$
is a diagonal matrix with diagonal entries taken independently at random from $\{-1,+1\}$.
For a sequence $(t_{1},...,t_{n})$ we denote by $\textbf{D}_{t_{1},...,t_{n}}$ a diagonal matrix with diagonal equal to $(t_{1},...,t_{n})$.
For a matrix $\textbf{A} \in \mathbb{R}^{n \times n}$ let $\|\textbf{A}\|_{F}$ denote its Frobenius and $\|A\|_{2} = \sup_{\textbf{x} \neq 0} \frac{\|\textbf{Ax}\|_{2}}{\|\textbf{x}\|_{2}}$ its spectral norm respectively. We denote by $\textbf{H}$ the $L_{2}$-normalized Hadamard matrix. We say that $\textbf{r}$ is a random Rademacher vector if every element of $\textbf{r}$ is chosen independently at random from $\{-1,+1\}$.

For a vector $\textbf{r} \in \mathbb{R}^{k}$ and $n>0$ let $\textbf{C}(\textbf{r},n) \in \mathbb{R}^{n \times nk}$ be a matrix,
where the first row is of the form $(\textbf{r}^{T},0,...,0)$ and each subsequent row is obtained from the previous one by right-shifting in a circulant manner the previous one by $k$. For a sequence of matrices $\textbf{W}^{1},...,\textbf{W}^{n} \in \mathbb{R}^{k \times n}$ we denote by $\textbf{V}(\textbf{W}^{1},...,\textbf{W}^{n}) \in \mathbb{R}^{nk \times n}$ a matrix obtained by stacking vertically matrices: $\textbf{W}^{1},...,\textbf{W}^{n}$.

Each structured matrix $\textbf{G}_{struct} \in \mathbb{R}^{n \times n}$ from the \textit{TripleSpin}-family is a product of three main structured components, i.e.:
\begin{equation}
\textbf{G}_{struct} = \textbf{M}_{3}\textbf{M}_{2}\textbf{M}_{1},
\end{equation}

where matrices $\textbf{M}_{1}, \textbf{M}_{2}$ and $\textbf{M}_{3}$ satisfy the following conditions:

\begin{framed}
\textbf{Condition 1:} Matrices: $\textbf{M}_{1}$ and $\textbf{M}_{2}\textbf{M}_{1}$ are $(\delta(n),p(n))$-balanced isometries. \\
\textbf{Condition 2:} $\textbf{M}_{2} = \textbf{V}(\textbf{W}^{1},...,\textbf{W}^{n})\textbf{D}_{\rho_{1},...,\rho_{n}}$ for some $(\Delta_{F},\Delta_{2})$-smooth set ${\textbf{W}^{1},...,\textbf{W}^{n}} \in \mathbb{R}^{k \times n}$ and some i.i.d sub-Gaussian random variables $\rho_{1},...,\rho_{n}$ with sub-Gaussian norm $K$. \\
\textbf{Condition 3:} $\textbf{M}_{3} = \textbf{C}(\textbf{r},n)$
for $\textbf{r} \in \mathbb{R}^{k}$, where $\textbf{r}$ is random Rademacher or Gaussian.
\end{framed}

If all three conditions are satisfied then we say that a matrix $\textbf{M}_{3}\textbf{M}_{2}\textbf{M}_{1}$ is a \textit{TripleSpin}-matrix with parameters: $\delta(n),p(n),K,\Lambda_{F},\Lambda_{2}$.
Below we explain these conditions. 

\begin{definition}[$(\delta(n),p(n))$-balanced matrices]
A randomized matrix $\textbf{M} \in \mathbb{R}^{n \times m}$ is $(\delta(n),p(n))$-balanced if for every $\textbf{x} \in \mathbb{R}^{m}$ with $\|\textbf{x}\|_{2} = 1$ we have: $\mathbb{P}[\|\textbf{M}\textbf{x}\|_{\infty} > \frac{\delta(n)}{\sqrt{n}}] \leq p(n)$.
\end{definition}

\begin{remark}
\label{balanceness_remark}
One can take as $\textbf{M}_{1}$ a matrix $\textbf{HD}_{1}$,
since as we will show in the Appendix, matrix $\textbf{HD}_{1}$ is $(\log(n),2ne^{-\frac{\log^{2}(n)}{8}})$-balanced. 
\end{remark}

\begin{definition}[$(\Delta_{F},\Delta_{2})$-smooth sets]
A deterministic set of matrices $\textbf{W}^{1},...,\textbf{W}^{n} \in \mathbb{R}^{k \times n}$ is $(\Lambda_{F},\Lambda_{2})$-smooth if:
\begin{itemize}
\item $\|\textbf{W}^{i}_{1}\|_{2} = .. = \|\textbf{W}^{i}_{n}\|_{2}$
      for $i = 1,...,n$, where $\textbf{W}^{i}_{j}$ stands for the  
      $j^{th}$ column of $\textbf{W}^{i}$,
 \item for $i \neq j$ and $l = 1,...,n$ we have: 
       $(\textbf{W}^{i}_{l})^{T} \cdot \textbf{W}^{j}_{l} = 0$, 
\item $\max_{i,j} \|(\textbf{W}^{j})^{T}\textbf{W}^{i}\|_{F} \leq \Lambda_{F}$ and $\max_{i,j} \|(\textbf{W}^{j})^{T}\textbf{W}^{i}\|_{2} \leq \Lambda_{2}.$
\end{itemize}
\end{definition}

\begin{remark}
If the unstructured matrix $\textbf{G}$ has rows taken from the general multivariate Gaussian distribution with diagonal covariance matrix $\Sigma \neq \textbf{I}$ then one needs to rescale vectors $\textbf{r}$ accordingly.
For clarity we assume here that $\Sigma = \textbf{I}$ and we present our theoretical results for that setting.
\end{remark}

All structured matrices previously considered are special cases of the $TripleSpin$-family (for clarity we will explicitly show it for some important special cases). Others, not considered before, are also covered by the $TripleSpin$-family. We have:

\begin{lemma}
\label{simple_lemma}
Matrices $\textbf{G}_{circ}\textbf{D}_{2}\textbf{HD}_{1}$, $\sqrt{n}\textbf{HD}_{3}\textbf{HD}_{2}\textbf{HD}_{1}$
and $\sqrt{n}\textbf{HD}_{g_{1},...,g_{n}}\textbf{HD}_{2}\textbf{HD}_{1}$, where $\textbf{G}_{circ}$ is Gaussian circulant, are valid $TripleSpin$-matrices for $\delta(n) = \log(n)$, $p(n) = 2ne^{-\frac{\log^{2}(n)}{8}}$, $K = 1$, $\Lambda_{F} = O(\sqrt{n})$ and $\Lambda_{2} = O(1)$. The same is true if one replaces $\textbf{G}_{circ}$ by a Gaussian Hankel or Toeplitz matrix.
\end{lemma}

\subsection{Stacking together \textit{TripleSpin}-matrices}

We described \textit{TripleSpin}-matrices as square matrices. In practice we are not restricted to square matrices. We can construct an $m \times n$ \textit{TripleSpin} matrix for $m \leq n$ from the square $n \times n$ \textit{TripleSpin}-matrix by taking its first $m$ rows. We can then stack vertically these independently constructed $m \times n$ matrices to obtain an $k \times n$ matrix for both: $k \leq n$ and $k > n$. We think about $m$ as another parameter of the model that tunes the "structuredness" level. Larger values of $m$ indicate more structured approach while smaller values lead to more random matrices (the $m=1$ case is the fully unstructured one).

\section{Computing general kernels with \textit{TripleSpin}-matrices}
Previous works regarding approximating kernels with structured matrices covered only some special kernels, namely: Gaussian, arc-cosine and angular kernels.
We explain here how structured approach (in particular our \textit{TripleSpin}-family) can be used to approximate well most kernels.
Theoretical guarantees that cover also this special case are given in the subsequent section.

For kernels that can be represented as an expectation of Gaussian random variables it is natural to approximate them using structured matrices. We start our analysis with the so-called \textit{Pointwise Nonlinear Gaussian kernels (PNG)} which are of the form:
\begin{equation}
\label{eq:main_function_simple}
\kappa_{f,\mu,\Sigma}(\textbf{x},\textbf{y}) = \mathbb{E}\left[f(\textbf{g}^\top \cdot \textbf{x})\cdot f(\textbf{g}^\top \cdot \textbf{y})\right]
\end{equation}
where: $\textbf{x},\textbf{y} \in \mathbb{R}^{n}$, the expectation is over a multivariate Gaussian $\textbf{g}\sim\mathcal{N}(\mu,\Sigma)$ and $f$ is a fixed nonlinear function $\mathbb{R} \rightarrow \mathbb{R}$ (the positive-semidefiniteness of the kernel follows from it being an expectation of dot products). 
The expectation, interpreted in the Monte-Carlo approximation setup, leads to an unbiased estimation by a sum of the normalized dot products $\frac{1}{k}f(\textbf{G} \textbf{x})^\top\cdot f(\textbf{G} \textbf{y})$, where $\textbf{G} \in \mathbb{R}^{k \times n}$ is a random Gaussian matrix and $f$ is applied pointwise, i.e. $f((v_{1},...,v_{k})):=(f(v_{1}),...,f(v_{k}))$. 

In this setting matrices from the $TripleSpin$-family can replace $\textbf{G}$ with diagonal covariance matrices $\Sigma$ to speed up computations and reduce storage complexity. 
This idea of using random structured matrices to evaluate such kernels is common in the literature, e.g. \cite{choromanska2015binary,le2013fastfood}, but no theoretical guarantees were given so far for general PNG kernels (even under the restriction of diagonal $\Sigma$).

PNGs define a large family of kernels characterized by $f,\mu,\Sigma$. 
Prominent examples include the Euclidean and angular distance \cite{choromanska2015binary}, the arc-cosine kernel \cite{NIPS2009_3628} and the sigmoidal neural network~\cite{Williams:1998:CIN:295919.295941}. 
Since sums of kernels are again kernels, we can construct an even larger family of kernels by summing PNGs. A simple example is the Gaussian kernel which can be represented as a sum of two PNGs with $f$ replaced by the trigonometric functions: $\sin$ and $\cos$.

Since the Laplacian, exponential and rational quadratic kernel can all be represented as a mixture of
Gaussian kernels with different variances,
they can be easily approximated by a finite sum of PNGs. Remarkably, virtually all kernels can be represented as a (potentially infinite) sum of PNGs with diagonal $\Sigma$.
Recall that kernel $\kappa$ is stationary if $\kappa(x,y)=\kappa(x-y)$ for all $x,y \in \mathbb{R}^{n}$ and non-stationary otherwise.
Harnessing Bochner's and Wiener’s Tauberian theorem~\cite{samo2015generalized}, we show that all stationary kernels may be approximated arbitrarily well by sums of PNGs. 
\begin{theorem}[stationary kernels] 
\label{thm:stationary}
The family of functions
\begin{align*}
\kappa_K(\textbf{x},\textbf{y})&:=\sum_{k=1}^K\alpha_k (\mathbb{E}[\cos(\textbf{g}_k^\top\cdot \textbf{x})\cos(\textbf{g}_k^\top\cdot \textbf{y})] + \mathbb{E}[\sin(\textbf{g}_k^\top\cdot \textbf{x})\sin(\textbf{g}_k^\top\cdot \textbf{y})])
\end{align*}
with $\textbf{g}_k\sim\mathcal{N}(\mu_k,\mbox{diag}((\sigma_{k}^1)^2,...,(\sigma_k^d)^2))$, $\mu_k,\sigma_k\in\mathbb{R}^{n}$, $\alpha_k\in\mathbb{R}$, $K\in\mathbb{N}\cup\{\infty\}$ is dense in the family of stationary real-valued kernels with respect to pointwise convergence.
\end{theorem}

This family corresponds to the \textit{spectral mixture kernels} of \cite{wilson2013gaussian}. We can extend these results to arbitrary, non-stationary, kernels; the precise statement and its proof can be found in the Appendix.

Theorem \ref{thm:stationary} and its non-stationary analogue show that the family of sums of PNGs contains virtually all kernel functions. If a kernel can be well approximated by the sum of only a small number of PNGs then we can use the $TripleSpin$-family to efficiently evaluate it. It has been found in the Gaussian Process literature that often a small sum of kernels tuned using training data is sufficient to explain phenomena remarkably well \cite{wilson2013gaussian}. The same should also be true for the sum of PNGs. This suggests a methodology for learning the parameters of a sum of PNGs from training data and then applying it out of sample. 
We leave this investigation for future research.

\section{Theoretical results}

We now show that matrices from the \textit{TripleSpin}-family can replace their unstructured counterparts in many machine learning algorithms with minimal loss of accuracy.

Let $\mathcal{A}$ be a randomized machine learning algorithm operating on a dataset $\mathcal{X} \subseteq \mathbb{R}^{n}$. We assume that $\mathcal{A}$ uses certain functions $f_{1},...,f_{s}$ such that each
$f_{i}$ uses a given Gaussian matrix $\textbf{G}$ (matrix $\textbf{G}$ can be the same or different for different functions) with independent rows taken from a multivariate Gaussian distribution with diagonal covariance matrix. $\mathcal{A}$ may also use other functions that do not depend directly on Gaussian matrices but may for instance operate on outputs of $f_{1},...,f_{s}$. We assume that each $f_{i}$ applies $\textbf{G}$ to vectors from some linear space $\mathcal{L}_{f_{i}}$ of dimensionality at most $d$.

\begin{remark}
In the kernel approximation setting with random feature maps one can match each pair of vectors $\textbf{x},\textbf{y} \in \mathcal{X}$ with a different $f=f_{\textbf{x},\textbf{y}}$.
Each $f$ computes the approximate value of the kernel for vectors $\textbf{x}$ and $\textbf{y}$.
Thus in that scenario $s = {|\mathcal{X}| \choose 2}$ and $d=2$ (since one can take: $\mathcal{L}_{f(\textbf{x},\textbf{y})} = span(\textbf{x},\textbf{y})$).
\end{remark}
\begin{remark}
In the vector quantization algorithms using random projection trees one can take $s=\nobreak 1$ (the algorithm $\mathcal{A}$ itself is a function $f$) and $d = d_{intrinsic}$, where $d_{intrinsic}$ is an intrinsic dimensionality of a given dataset $\mathcal{X}$ (random projection trees are often used if $d_{intrinsic} \ll n$).
\end{remark}

Fix a function $f$ of $\mathcal{A}$ using an unstructured Gaussian matrix $\textbf{G}$ and applying it on the linear space $\mathcal{L}_{f}$ of dimensionality $l \leq d$. 
Note that the outputs of $\textbf{G}$
on vectors from $\mathcal{L}_{f}$ are determined by the sequence: 
$(\textbf{G}\textbf{x}^{1},...,\textbf{G}\textbf{x}^{l})$,
where $\textbf{x}^{1},...,\textbf{x}^{l}$ stands for a fixed orthonormal basis of $\mathcal{L}_{f}$. Thus they are determined by a following vector (obtained from "stacking together" vectors $\textbf{Gx}^{1}$,...,$\textbf{Gx}^{l}$):
$$\textbf{q}_{f} = ((\textbf{G}\textbf{x}^{1})^{T},...,(\textbf{G}\textbf{x}^{l})^{T})^{T}.$$ 

Notice that $\textbf{q}_{f}$ is a Gaussian vector with independent entries (this comes from the fact that rows of $\textbf{G}$ are independent and the observation that the projections of the Gaussian vector on the orthogonal directions are independent). Thus the covariance matrix of $\textbf{q}_{f}$ is an identity.

\begin{definition}[$\epsilon$-similarity]
We say that a multivariate Gaussian distribution $\textbf{q}_{f}(\epsilon)$ is $\epsilon$-close to the multivariate Gaussian distribution $\textbf{q}_{f}$ with covariance matrix $\textbf{I}$ if its covariance matrix is equal to $1$ on the diagonal and has all other entries of absolute value at most $\epsilon$.
\end{definition}

To measure the "closeness" of the algorithm $\mathcal{A}^{\prime}$ with the \textit{TripleSpin}-model matrices to $\mathcal{A}$, we will measure how "close" the corresponding vector $\textbf{q}^{\prime}_{f}$ for $\mathcal{A}^{\prime}$ is to $\textbf{q}_{f}$ in the distribution sense.
The following definition proposes a quantitative way to measure this
closeness.
Without loss of generality we will assume now that each structured matrix consists of just one block since different blocks of the structured matrix are chosen independently.

\begin{definition}
\label{closeness}
Let $\mathcal{A}$ be as above. For a given $\eta, \epsilon > 0$ the class of algorithms $\mathcal{A}_{\eta, \epsilon}$ is given as a set of algorithms $\mathcal{A}^{\prime}$ obtained from $\mathcal{A}$ by replacing unstructured Gaussian matrices with their structured counterparts such that for any $f$ with $dim(\mathcal{L})_{f} = l$ and any convex set $S \in \mathbb{R}^{ml}$ the following holds:
\begin{equation}
|\mathbb{P}[\textbf{q}_{f}(\epsilon) \in S] - \mathbb{P}[\textbf{q}^{\prime}_{f} \in S]| \leq \eta,
\end{equation}
where $\textbf{q}_{f}(\epsilon)$ is some multivariate Gaussian distribution that is $\epsilon$-similar to $\textbf{q}_{f}$.
\end{definition}

The smaller $\epsilon$, the closer in distribution are $\textbf{q}_{f}$ and $\textbf{q}_{f}^{\prime}$ and the more accurate the structured version $\mathcal{A}^{\prime}$ of $\mathcal{A}$ is.
Now we show that \textit{TripleSpin}-matrices lead to algorithms from $\mathcal{A}_{\eta, \epsilon}$ with $\eta, \epsilon \ll 1$.

\begin{theorem}[structured ml algorithms]
\label{main_struct_theorem}
Let $\mathcal{A}$ be the randomized algorithm using unstructured Gaussian matrices $\textbf{G}$ and let $s,d$ be as at the beginning of the section.
Replace the unstructured matrix $\textbf{G}$ by one of 
its structured variants from the \textit{TripleSpin}-family defined in Section \ref{sec:model} with blocks of $m$ rows each. 
Then for $n$ large enough and $\epsilon = o_{md}(1)$ with probability at least:

\begin{equation}
1 - 2p(n)sd - 2{md \choose 2}se^{-\Omega(\min(\frac{\epsilon^{2}n^{2}}{K^{4}\Lambda_{F}^{2}\delta^{4}(n)}, \frac{\epsilon n}{K^{2}\Lambda_{2} \delta^{2}(n)}))}
\end{equation}
the structured version of the algorithm belongs to the class $\mathcal{A}_{\eta, \epsilon}$, where: $\eta = \frac{\delta^{3}(n)}{n^{\frac{2}{5}}}$, $\delta(n), p(n),K, \Lambda_{F}, \Lambda_{2}$ are as in the definition of the \textit{TripleSpin}-family from Section \ref{sec:model}
and the probability is in respect to the random choices of $\textbf{M}_{1}$
and $\textbf{M}_{2}$.
\end{theorem}

Theorem \ref{main_struct_theorem} implies strong accuracy guarantees for the specific matrices from the $TripleSpin$-family. As a corollary we get for instance:

\begin{theorem}
\label{corollary_theorem}
Under assumptions from Theorem~\ref{main_struct_theorem}
the probability that the structured version of the algorithm belongs to $\mathcal{A}_{\eta, \epsilon}$ for $\eta = \frac{\delta^{3}(n)}{n^{\frac{2}{5}}}$ is at least:
$1 - 4ne^{-\frac{\log^{2}(n)}{8}}sd - 2{md \choose 2}s e^{-\Omega(\frac{\epsilon^{2}n}{\log^{4}(n)})}$
for the structured matrices $\sqrt{n}\textbf{HD}_{3}\textbf{HD}_{2}\textbf{HD}_{1}$, $\sqrt{n}\textbf{HD}_{g_{1},...,g_{n}}\textbf{HD}_{2}\textbf{HD}_{1}$ as well as for the structured matrices of the form $\textbf{G}_{struct}\textbf{D}_{2}\textbf{HD}_{1}$, where $\textbf{G}_{struct}$ is Gaussian circulant, Gaussian Toeplitz or Gaussian Hankel matrix.
\end{theorem}


As a corollary of Theorem~\ref{corollary_theorem}, we obtain the following result showing the effectiveness of the cross-polytope LSH with structured matrices $\textbf{HD}_{3}\textbf{HD}_{2}\textbf{HD}_{1}$ that was only heuristically confirmed before~\cite{indyk2015}.

\begin{theorem}
\label{hopefully_last_theorem}
Let $\textbf{x},\textbf{y} \in \mathbb{R}^{n}$ be two unit $L_{2}$-norm vectors. 
Let $\textbf{v}_{\textbf{x}, \textbf{y}}$ be the vector indexed by all $m^{2}$ ordered pairs of canonical directions $(\pm \textbf{e}_{i},\pm \textbf{e}_{j})$, where the value of the entry indexed by $(\textbf{u},\textbf{w})$ is the probability that: $h(\textbf{x}) = \textbf{u}$ and $h(\textbf{y}) = \textbf{w}$, and $h(\textbf{v})$ stands for the hash of $\textbf{v}$. Then with probability at least: 
$p_{success} =  1 - 4ne^{-\frac{\log^{2}(n)}{8}}sd - 2{md \choose 2}s e^{-\Omega(\frac{\epsilon^{2}n}{\log^{4}(n)})}$
the version of the stochastic vector $\textbf{v}^{1}_{\textbf{x},\textbf{y}}$
for the unstructured Gaussian matrix $\textbf{G}$ and its structured counterpart $\textbf{v}^{2}_{\textbf{x},\textbf{y}}$ for the matrix $\textbf{HD}_{3}\textbf{HD}_{2}\textbf{HD}_{1}$ satisfy: 
$\|\textbf{v}^{1}_{\textbf{x},\textbf{y}} - \textbf{v}^{2}_{\textbf{x},\textbf{y}}\|_{\infty} \leq \log^{3}(n)/n^{\frac{2}{5}} + c \epsilon$,
for $n$ large enough, where $c > 0$ is a universal constant.
The probability above is taken with respect to random choices of $\textbf{D}_{1}$ and $\textbf{D}_{2}$.
\end{theorem}

The proof for the discrete structured setting applies Berry-Esseen-type results for random vectors (details in the Appendix) showing that for $n$ large enough $\pm 1$ random vectors $\textbf{r}$ act similarly to Gaussian vectors.

\section{Experiments}
\label{sec:experiments}

Experiments have been carried out on a single processor machine (Intel Core i7-5600U CPU @ 2.60GHz, 4 hyper-threads) with 16GB RAM for the first two applications and a dual processor machine (Intel Xeon E5-2640 v3 @ 2.60GHz, 32 hyper-threads) with 128GB RAM for the last one. Every experiment was conducted using Python. In particular, NumPy is linked against a highly optimized BLAS library (Intel MKL). Fast Fourier Transform is performed using numpy.fft and Fast Hadamard Transform is using ffht from~\cite{indyk2015}. To have fair comparison, we have set up: $\mathrm{OMP\_NUM\_THREADS}=1$ so that every experiment is done on a single thread. Every parameter corresponding to a \textit{TripleSpin}-matrix is computed in advance, such that obtained speedups take only matrix-vector products into account. 
All figures should be viewed in color.

\subsection{Locality-Sensitive Hashing (LSH) application} 

In the first experiment, to support our theoretical results for the \textit{TripleSpin}-matrices in the cross-polytope LSH, we compared collision probabilities in Figure \ref{fig:collision} for the low dimensional case. Results are shown for one hash function (averaged over $100$ runs). For each interval, collision probability has been computed for $20 000$ points for a random Gaussian matrix $\textbf{G}$ and four other types of \textit{TripleSpin}-matrices (descending order of number of parameters): $\textbf{G}_{Toeplitz}\textbf{D}_{2}\textbf{HD}_{1}$, $\textbf{G}_{skew-circ}\textbf{D}_{2}\textbf{HD}_{1}$, $\textbf{HD}_{g_{1},...,g_{n}}\textbf{HD}_{2}\textbf{HD}_{1}$, and $\textbf{HD}_{3}\textbf{HD}_{2}\textbf{HD}_{1}$, where $\textbf{G}_{Toeplitz}$, and $\textbf{G}_{skew-circ}$ are respectively Gaussian Toeplitz and Gaussian skew-circulant matrices. 

We can see that all \textit{TripleSpin}-matrices show high collision probabilities for small distances and low ones for large distances. All the curves are almost identical. As theoretically predicted, there is no loss of accuracy (sensitivity) by using matrices from the \textit{TripleSpin}-family.

\begin{figure}[!ht]
\CenterFloatBoxes
\begin{floatrow}
\ffigbox
  {\includegraphics[scale = 0.16]{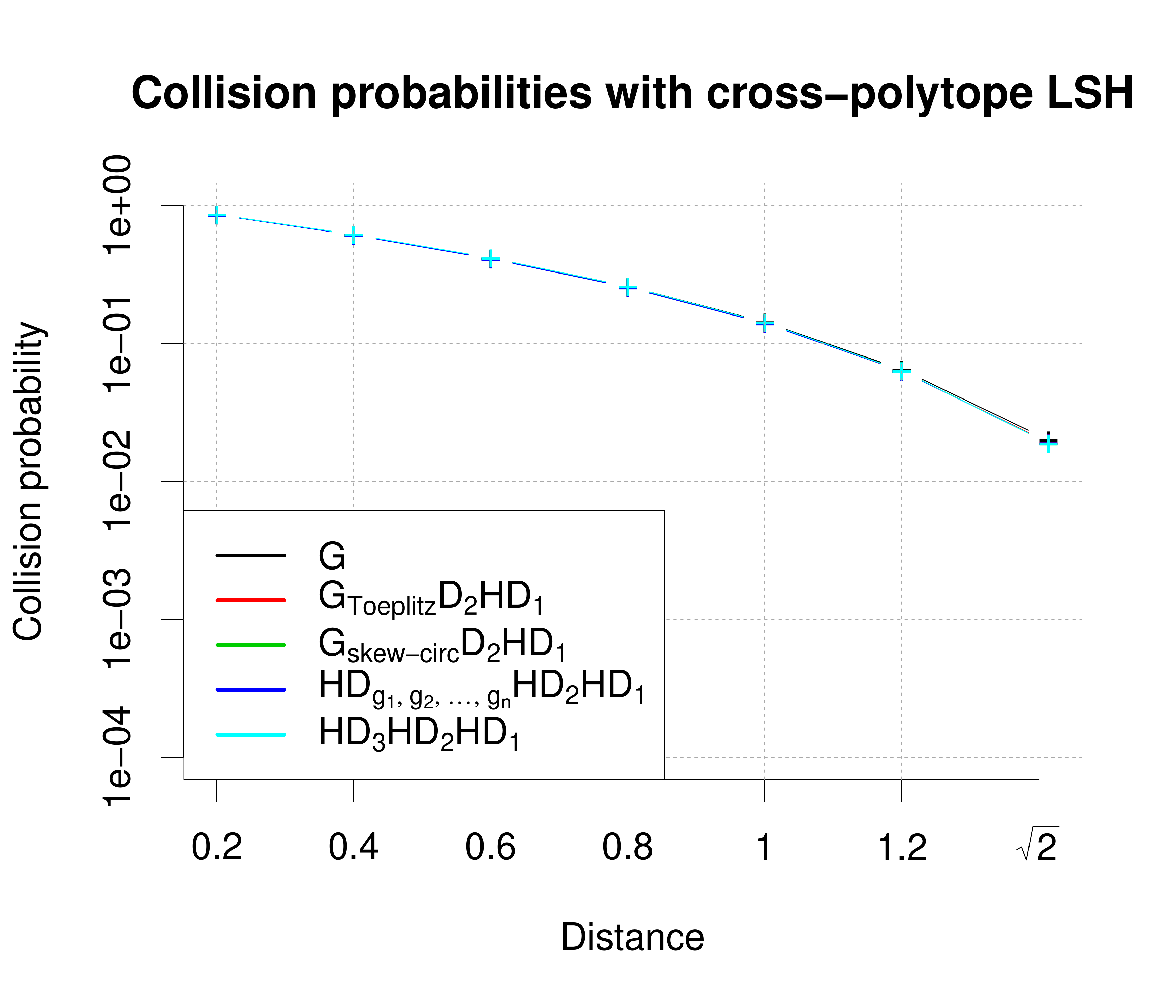}}
  {\caption{Cross-polytope LSH - collision  probabilities}\label{fig:collision}}
\killfloatstyle
\vspace{0pt}
\ttabbox
  {\renewcommand{\arraystretch}{1.5} 
  \resizebox{1.0\linewidth}{!}{%
  \begin{tabular}{|l|l|l|l|l|l|l|l|}
\hline
Matrix dimensions                                             & $2^{9}$ & $2^{10}$ & $2^{11}$ & $2^{12}$ & $2^{13}$ & $2^{14}$ & $2^{15}$ \\ \hline 

$\textbf{G}_{Toeplitz}\textbf{D}_{2}\textbf{HD}_{1}$             & x1.4      & x3.4       & x6.4       & x12.9       & x28.0       & x42.3       & x89.6       \\ \hline

$\textbf{G}_{skew-circ}\textbf{D}_{2}\textbf{HD}_{1}$         & x1.5      & x3.6       & x6.8       & x14.9       & x31.2       & x49.7       & x96.5       \\ \hline

$\textbf{HD}_{g_{1},...,g_{n}}\textbf{HD}_{2}\textbf{HD}_{1}$ & x2.3      & x6.0       & x13.8       & x31.5       & x75.7       & x137.0      & x308.8       \\ \hline

$\textbf{HD}_{3}\textbf{HD}_{2}\textbf{HD}_{1}$               & x2.2      & x6.0       & x14.1       & x33.3       & x74.3       & x140.4       & x316.8       \\ \hline
\end{tabular}
}%
  }
  {\vspace{0.25cm}
\caption{Speedups for Gaussian kernel approximation}
\label{tab:speedupskernel}
\vspace{-0.25cm}}
\end{floatrow}
\end{figure}






\subsection{Kernel approximation}

In the second experiment, we compared feature maps obtained with Gaussian random matrices and specific \textit{TripleSpin}-matrices for Gaussian and angular kernels. To test the quality of the structured kernels' approximations, we compute the corresponding Gram-matrix reconstruction error using the Frobenius norm metric as in \cite{chor_sind_2016} : $ \frac{||\textbf{K} - \tilde{\textbf{K}} ||_{F}}{||\textbf{K}||_{F}} $, where $\textbf{K}, \tilde{\textbf{K}}$ are respectively the exact and approximate Gram-matrices, as a function of the number of random features. When number of random features $k$ is greater than data dimensionality $n$, we apply described block-mechanism.
We used the USPST dataset (test set) which consists of scans of handwritten digits from envelopes by the U.S. Postal Service. It contains 2007 points of dimensionality 258 ($n = 258$) corresponding to descriptors of 16 x 16 grayscale images. For Gaussian kernel, bandwidth $\sigma$ is set to $9.4338$. The results are averaged over $10$ runs. 

\paragraph{Results on the USPST dataset:} The following matrices have been tested: Gaussian random matrix $\textbf{G}$,  $\textbf{G}_{Toeplitz}\textbf{D}_{2}\textbf{HD}_{1}$ $\textbf{G}_{skew-circ}\textbf{D}_{2}\textbf{HD}_{1}$, $\textbf{HD}_{g_{1},...,g_{n}}\textbf{HD}_{2}\textbf{HD}_{1}$ and $\textbf{HD}_{3}\textbf{HD}_{2}\textbf{HD}_{1}$.

In Figure \ref{kernel_approx_USPST}, for both kernels, all \textit{TripleSpin}-matrices perform similarly to a random Gaussian matrix, but $\textbf{HD}_{3}\textbf{HD}_{2}\textbf{HD}_{1}$ is giving the best results (see Figure \ref{kernel_approx} in the Appendix for additional experiments). The efficiency of the \textit{TripleSpin}-mechanism does not depend on the dataset.

\begin{figure}[!t]
\centering
\includegraphics[scale = 0.19]{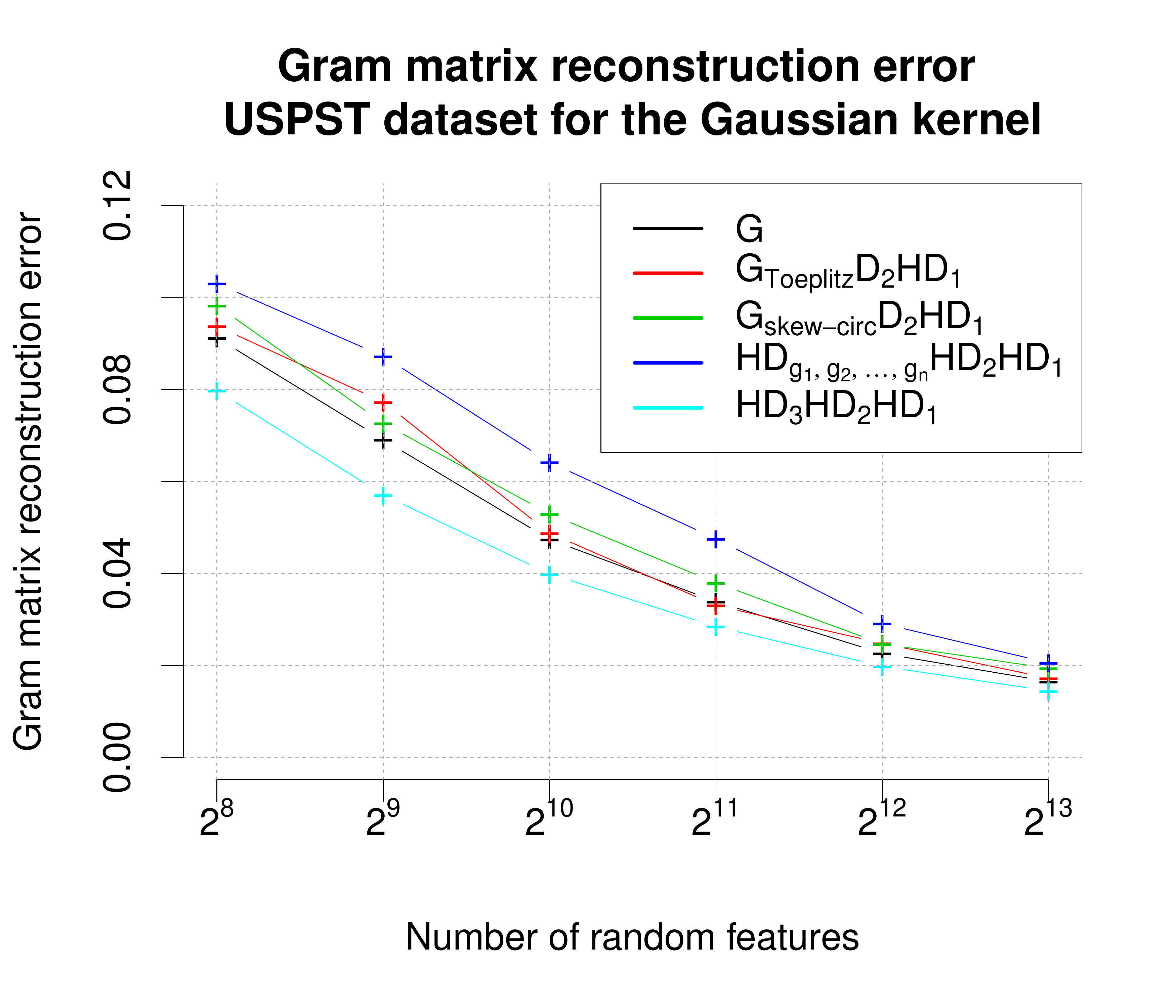}
\includegraphics[scale = 0.19]{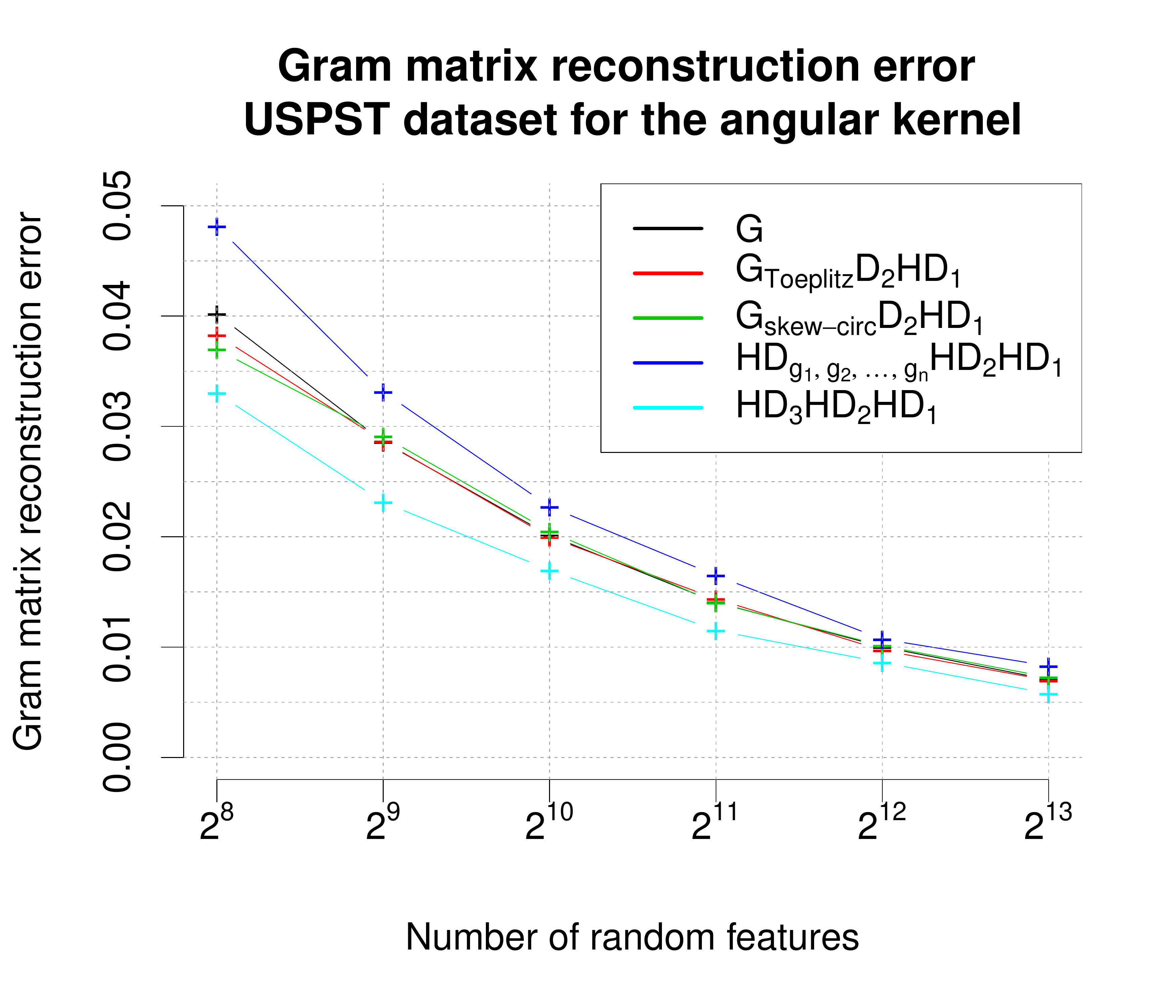}
\caption{Accuracy of random feature map kernel approximation}
\label{kernel_approx_USPST}
\end{figure}

Table \ref{tab:speedupskernel} shows significant speedups
obtained by the \textit{TripleSpin}-matrices. Those are defined as $\mbox{time}(\textbf{G})/\mbox{time}(\textbf{T})$, where $\textbf{G}$ is a random Gaussian matrix, $\textbf{T}$ is a \textit{TripleSpin} matrix
and $\mbox{time}(\textbf{X})$ stands for the corresponding runtime.

\subsection{Newton sketches}
Our last experiment covers the Newton sketch approach initially proposed in~\cite{pilanci} as 
a generic optimization framework. In the subsequent experiment we show  that {\em TripleSpin} matrices can be used for this purpose, thus can speed up several convex optimization problems solvers. The logistic regression problem is considered (see the Appendix for more details). Our goal is to find $x \in \mathbb{R}^d$, which minimizes the logistic regression cost, given a dataset $\{(a_i,y_i)\}_{i=1..n}$, with $a_i \in 
\mathbb{R}^d$ sampled according to a Gaussian centered multivariate distribution with covariance $\Sigma_{i,j} = 0.99^{|i-j|}$ and $y_i \in \{-1,1\}$, generated at random. Various sketching matrices $S^t \in \mathbb{R}^{m \times n}$ are considered. 

We first show that equivalent convergence properties are exhibited for various {\em TripleSpin}-matrices. Figure~\ref{fig:newtonSketchConvergence} illustrates the convergence of the Newton sketch algorithm, as measured by the optimality gap defined in~\cite{pilanci}, versus the iteration number. While it is clearly expected that sketched versions of the algorithm do not converge as quickly as the exact Newton-sketch approach, the figure confirms that various {\em TripleSpin}-matrices exhibit similar convergence behavior. 

As shown in~\cite{pilanci}, when the dimensionality of the problem increases, the computational cost of computing the Hessian in the exact Newton-sketch approach becomes very large, scaling as $\mathcal{O}(nd^2)$. The complexity of the structured Newton-sketch approach with \textit{TripleSpin}-matrices is only $\mathcal{O}(d n \log(n) + md^2)$. Wall-clock times of computing single Hessian matrices are illustrated in Figure~\ref{fig:newtonSketchConvergence}. This figure confirms that the increase in number of iterations of the Newton sketch compared to the exact sketch is compensated by the efficiency of sketched computations, in particular Hadamard-based sketches yield improvements at the lowest dimensions.

\begin{figure}[htb]
\centering
\begin{subfigure}[b]{.5\linewidth}
\centering
\includegraphics[width=\columnwidth]{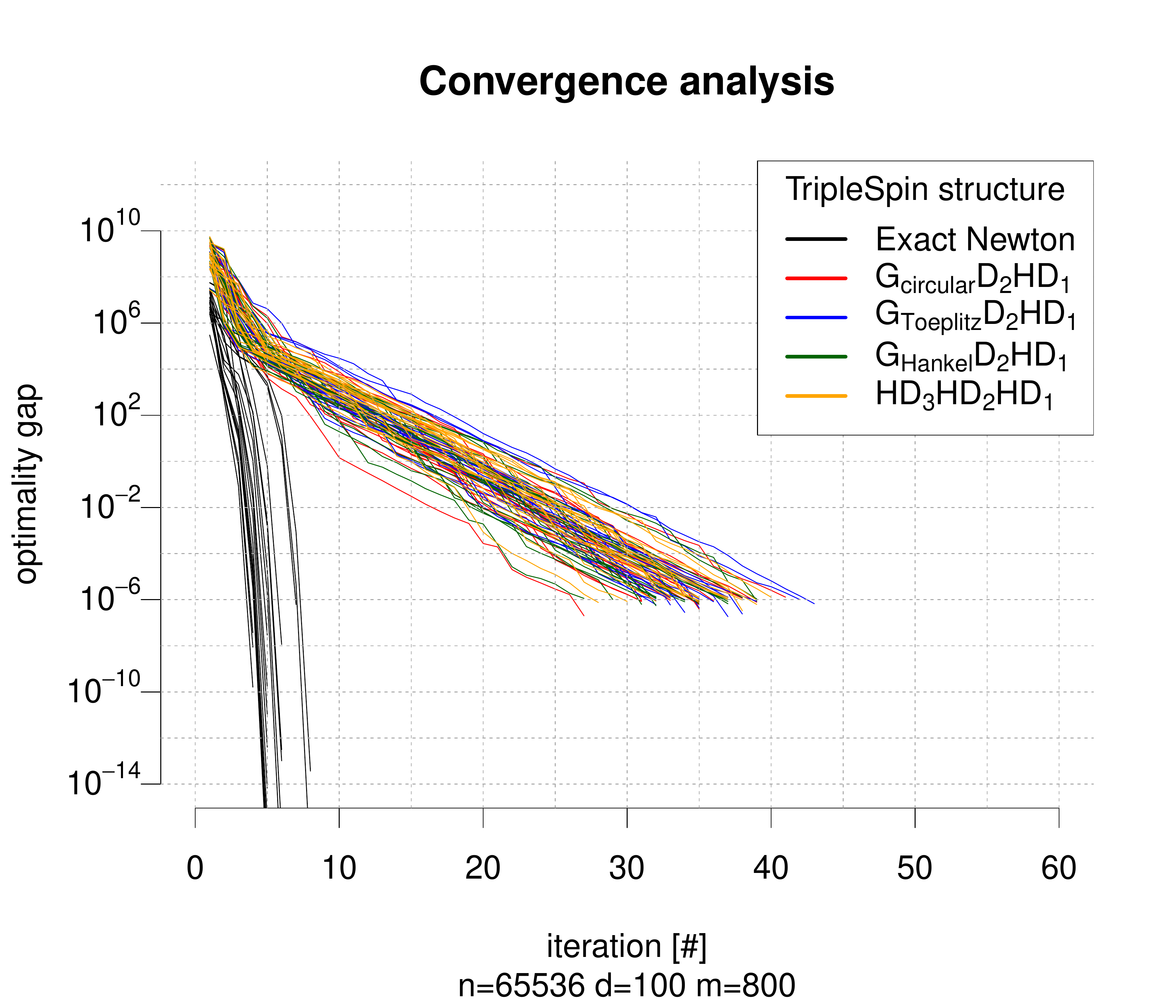}
\label{fig:newtonSketcha}
\end{subfigure}%
\begin{subfigure}[b]{.5\linewidth}
\centering
\includegraphics[width=\columnwidth]{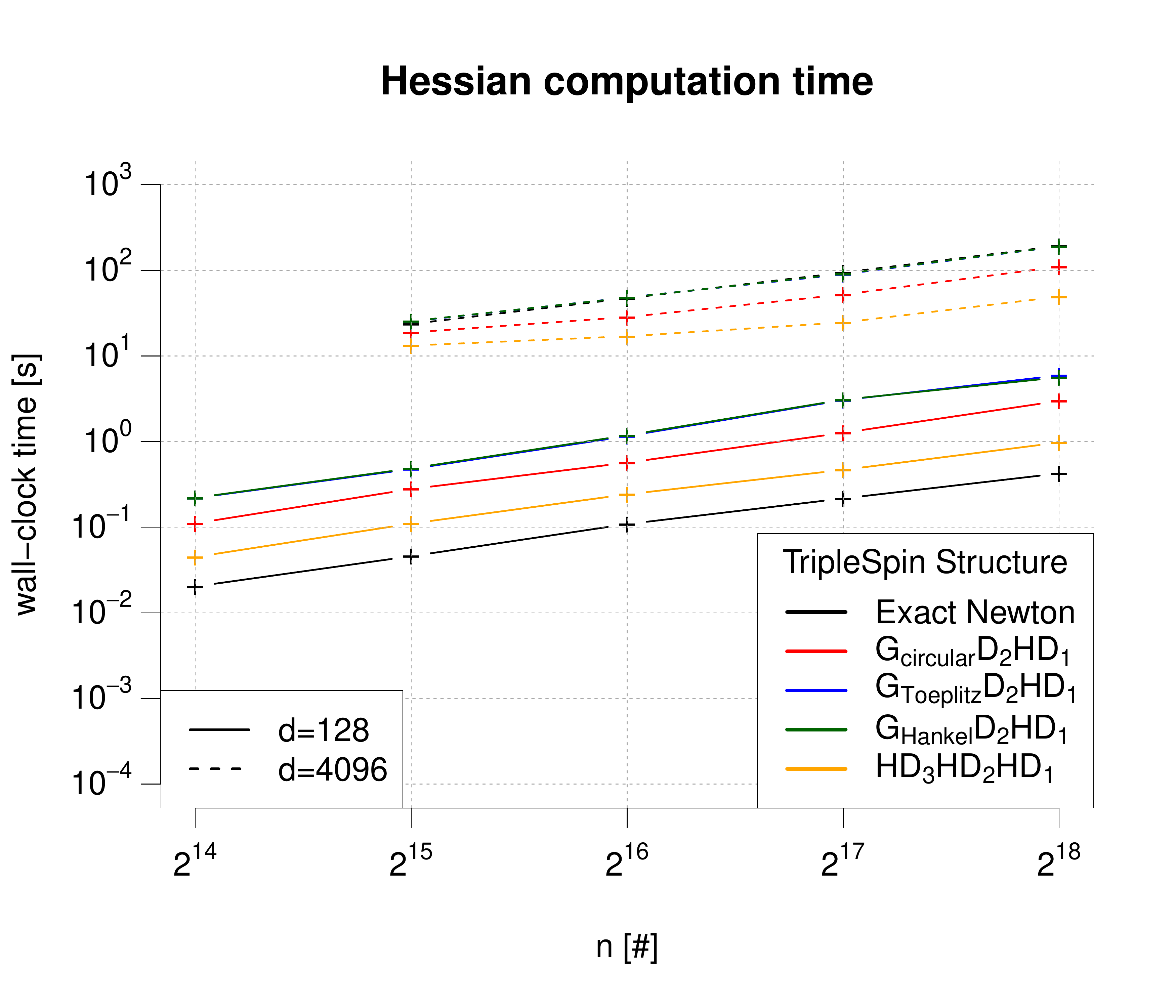}
\label{fig:newtonSketchb}
\end{subfigure}%
\caption{Numerical illustration of convergence (left) and computational complexity (right) of the Newton sketch algorithm with various {\em TripleSpin}-matrices. (left) Various sketching structures are compared in terms of convergence against iteration number. (right) Wall-clock times of {\em TripleSpin} structures are compared in various dimensionality settings.}
\label{fig:newtonSketchConvergence}
\end{figure}

\newpage
\small{
\bibliographystyle{unsrt}
\bibliography{three_struct_block}
}

\newpage
\section{Appendix}

In the Appendix we prove all theorems presented in the main body of the paper.
\subsection{Computing general kernels with \textit{TripleSpin}-matrices}

We prove here Theorem \ref{thm:stationary} and its non-stationary analogue.
For the convenience of the Reader, we restate both theorems.
We start with Theorem \ref{thm:stationary} that we restate as follows.

\textbf{Theorem \ref{thm:stationary} (stationary kernels)}
The family of functions
\begin{align*}
\kappa_K(\textbf{x},\textbf{y})&:=\sum_{k=1}^K\alpha_k (\mathbb{E}[\cos(\textbf{g}_k^\top\cdot \textbf{x})\cos(\textbf{g}_k^\top\cdot \textbf{y})] + \mathbb{E}[\sin(\textbf{g}_k^\top\cdot \textbf{x})\sin(\textbf{g}_k^\top\cdot \textbf{y})])
\end{align*}
with $\textbf{g}_k\sim\mathcal{N}(\mu_k,\mbox{diag}((\sigma_k^1)^2,...,(\sigma_k^d)^2))$, $\mu_k,\sigma_k\in\mathbb{R}^{n}$, $\alpha_k\in\mathbb{R}$, $K\in\mathbb{N}\cup\{\infty\}$ is dense in the family of stationary real-valued kernels with respect to pointwise convergence.


\begin{proof}
Theorem 3 of \cite{samo2015generalized} states that:

``Let $h$ be a real-valued positive semi-definite, continuous, and integrable function such that $\forall\tau\in\mathbb{R}^n,h(\tau)>0$. The family of functions
\begin{equation*}
\kappa_K(\tau):=\sum_{k=1}^K\alpha_k h(\tau\odot\sigma_k)\cos(2\pi\mu_k\tau)
\end{equation*}
with $\mu_k,\sigma_k\in\mathbb{R}^{+n}$, $\alpha_k\in\mathbb{R}$, $K\in\mathbb{N}\cup\{\infty\}$ is dense in the family of stationary real-valued kernels with respect to pointwise convergence.'' Here $\mathbb{R}^{+n}=\{x\in\mathbb{R}^{n}: x_i\geq 0\mbox{ for all }i=1,...,n \}\subset\mathbb{R}^{n}$.

Let $\odot$ denote the element-wise product.
If we choose $h(\tau)=\exp(-2\pi^2\|\tau\|^2)$, as suggested in \cite{samo2015generalized}, then it follows that
\begin{align}
\kappa_K(\tau)&:=\sum_{k=1}^K\alpha_k \exp(-2\pi^2\|\sigma_k\odot\tau\|^2)\cos(2\pi\mu_k^\top\cdot \tau) \notag\\
&=\sum_{k=1}^K\alpha_k \exp(-0.5(2\pi\|\sigma_k\odot\tau\|)^2)\cos(2\pi\mu_k^\top\cdot \tau) \notag\\
&=\sum_{k=1}^K\alpha_k  \int_{\mathbb{R}} \mathcal{N}(g;0,1)\cos(2\pi\|\sigma_k\odot\tau\|g)\cos(2\pi\mu_k^\top\cdot \tau)dg \label{eq:t13}\\
&=\sum_{k=1}^K\alpha_k  \int_{\mathbb{R}} \mathcal{N}(g;0,\|\sigma_k\odot\tau\|^2)\cos(2\pi g)\cos(2\pi\mu_k^\top\cdot \tau)dg \label{eq:t14}\\
&=\sum_{k=1}^K\alpha_k \int_{\mathbb{R}} \mathcal{N}(g;0,\|\sigma_k\odot\tau\|^2)[\cos(2\pi g)\cos(2\pi\mu_k^\top\cdot \tau)-\sin(2\pi g)\sin(2\pi\mu_k^\top\cdot \tau)]dg \label{eq:t15}\\
&=\sum_{k=1}^K\alpha_k \int_{\mathbb{R}} \mathcal{N}(g;0,\|\sigma_k\odot\tau\|^2)\cos(2\pi (g+\mu_k^\top\cdot \tau))dg \label{eq:t16}\\
&=\sum_{k=1}^K\alpha_k \int_{\mathbb{R}} \mathcal{N}(g;0,\tau^\top\mbox{diag}((\sigma_1^k)^2,...,(\sigma_d^k)^2)\tau)\cos(2\pi (g+\mu_k^\top\cdot \tau))dg \notag\\
&=\sum_{k=1}^K\alpha_k \int_{\mathbb{R}^n} \mathcal{N}(\textbf{g};0,\mbox{diag}((\sigma_1^k)^2,...,(\sigma_d^k)^2))\cos(2\pi (\textbf{g}^\top\cdot \tau+\mu_k^\top\cdot \tau))d\textbf{g} \label{eq:t18}\\
&=\sum_{k=1}^K\alpha_k \int_{\mathbb{R}^n} \mathcal{N}(\textbf{g};\mu_k,\mbox{diag}((\sigma_1^k)^2,...,(\sigma_d^k)^2))\cos(2\pi \textbf{g}^\top\cdot \tau)d\textbf{g} \notag\\
&=\sum_{k=1}^K\alpha_k \mathbb{E}[\cos(2\pi \textbf{g}_k^\top\cdot\tau)] \notag\\
&=\sum_{k=1}^K\alpha_k \mathbb{E}[\cos(2\pi \textbf{g}_k^\top\cdot(\textbf{x}-\textbf{y}))] \notag\\
&=\sum_{k=1}^K\alpha_k (\mathbb{E}[\cos(2\pi \textbf{g}_k^\top\cdot \textbf{x})\cos(2\pi \textbf{g}_k^\top\cdot \textbf{y})] + \mathbb{E}[\sin(2\pi \textbf{g}_k^\top\cdot \textbf{x})\sin(2\pi \textbf{g}_k^\top\cdot \textbf{y})]) \notag
\end{align}
is dense in the family of stationary real-valued kernels with respect to pointwise convergence.
Equation~(\ref{eq:t13}) follows from Bochner's theorem, (\ref{eq:t14}) from integration by substitution,
(\ref{eq:t15}) since sine is an odd function, (\ref{eq:t16}) from cosine angle sum identity,
(\ref{eq:t18}) from writing $g=\tau^\top\cdot\textbf{g}$ as linear transform of $\textbf{g}$. Absorbing $2\pi$ into $\mu_k$ and $\sigma_k$ and relaxing $\mu_k,\sigma_k\in \mathbb{R}^{+n}$ to $\mu_k,\sigma_k\in\mathbb{R}^{n}$ completes the proof.
\end{proof}





Now we will show the analogous version of that result for non-stationary kernels.

\begin{theorem}[non-stationary kernels]
The family of functions
\begin{equation*}
\kappa(\textbf{x},\textbf{y}) = \sum_{k=1}^K\alpha_k( \mathbb{E}[\cos( \textbf{g}_k^\top\cdot \textbf{x})\cos(\textbf{g}_k^\top\cdot \textbf{y})]+\mathbb{E}[\sin( \textbf{g}_k^\top\cdot \textbf{x})\sin( \textbf{g}_k^\top\cdot \textbf{y})]) \Psi_k(\textbf{x})^\top\Psi_k(\textbf{y})
\end{equation*}
where $\Psi_k(\textbf{x})=\left( \begin{array}{c}
\cos(\textbf{x}^\top\cdot \textbf{w}_k^1) + \cos( \textbf{x}^\top\cdot \textbf{w}_k^2)\\
\sin(\textbf{x}^\top\cdot \textbf{w}_k^1) + \sin( \textbf{x}^\top\cdot \textbf{w}_k^2) \end{array} \right)$, with $\textbf{g}_k\sim\mathcal{N}(0,\mbox{diag}((\sigma_1^k)^2,...,(\sigma_d^k)^2)), \sigma_k\in\mathbb{R}^{n}, \textbf{w}_k^1,\textbf{w}_k^2\in\mathbb{R}^n,\alpha_k\in\mathbb{R},K\in\mathbb{N}\cup\{\infty\}$ is dense in the family of real-valued continuous bounded non-stationary kernels with respect to the pointwise convergence of functions.
\end{theorem}


\begin{proof}
Theorem 7 of \cite{samo2015generalized} states that:

``Let $(\textbf{x}, \textbf{y}) \to \kappa^*(\textbf{x},\textbf{y})$ be a real-valued, positive semi-definite, continuous, and integrable function such that $\forall \textbf{x},\textbf{y},\kappa^*(\textbf{x},\textbf{y})>0$. The family
\begin{equation*}
\kappa(\textbf{x},\textbf{y}) = \sum_{k=1}^K\alpha_k \kappa^*(\textbf{x}\odot \sigma_k,\textbf{y}\odot \sigma_k) \Psi_k(\textbf{x})^\top\Psi_k(\textbf{y})
\end{equation*}
where $\Psi_k(\textbf{x})=\left( \begin{array}{c}
\cos(2\pi \textbf{x}^\top\cdot \textbf{w}_k^1) + \cos(2\pi \textbf{x}^\top\cdot \textbf{w}_k^2)\\
\sin(2\pi \textbf{x}^\top\cdot \textbf{w}_k^1) + \sin(2\pi \textbf{x}^\top\cdot \textbf{w}_k^2) \end{array} \right)$, with $\sigma_k\in\mathbb{R}^{n+}, \textbf{w}_k^1,\textbf{w}_k^2\in\mathbb{R}^n,\alpha_k\in\mathbb{R},K\in\mathbb{N}\cup\{\infty\}$ is dense in the family of real-valued continuous bounded non-stationary kernels with respect to the pointwise convergence of functions.''

If we choose as $\kappa^*$ the Gaussian kernel:
\begin{align*}
\kappa^*(\textbf{x},\textbf{y})&=\exp(-\|\textbf{x}-\textbf{y}\|^2/2)\\
&=\mathbb{E}[\cos(\textbf{g}^\top\cdot(\textbf{x}-\textbf{y}))]
\end{align*}
with $g\sim\mathcal{N}(0,I)$ then
\begin{align*}
\kappa^*(\textbf{x}\odot \sigma_k,\textbf{y}\odot \sigma_k)&=\mathbb{E}[\cos(\textbf{g}_k^\top\cdot(\textbf{x}-\textbf{y}))]\\
\end{align*}
with $\textbf{g}_k\sim\mathcal{N}(0,\mbox{diag}((\sigma_1^k)^2,...,(\sigma_d^k)^2))$.  Absorbing $2\pi$ into $w_k$ and relaxing $\sigma_k\in \mathbb{R}^{+n}$ to $\sigma_k\in\mathbb{R}^{n}$ completes the proof.
\end{proof}

\subsection{Structured machine learning algorithms with \textit{TripleSpin}-matrices}

We prove now Lemma \ref{simple_lemma}, Remark \ref{balanceness_remark},
as well as Theorem \ref{main_struct_theorem} and Theorem \ref{corollary_theorem}.

\subsubsection{Proof of Remark \ref{balanceness_remark}}

This result first appeared in \cite{ailon2006approximate}. The following proof was given in \cite{chor_sind_2016}, we repeat it here for completeness.
We will use the following standard concentration result.

\begin{lemma}(Azuma's Inequality)
Let $X_{1},...,X_{n}$ be a martingale and assume that $-\alpha_{i} \leq X_{i} \leq \beta_{i}$ for some positive constants $\alpha_{1},...,\alpha_{n}, \beta_{1},...,\beta_{n}$. 
Denote $X = \sum_{i=1}^{n} X_{i}$.
Then the following is true:
\begin{equation}
\mathbb{P}[|X - \mathbb{E}[X]| > a] \leq 2e^{-\frac{a^{2}}{2\sum_{i=1}^{n}(\alpha_{i} + \beta_{i})^{2}}}
\end{equation}
\end{lemma}

\begin{proof}
Denote by $\tilde{\textbf{x}}^{j}$ an image of $\textbf{x}^{j}$ under transformation $\textbf{HD}$. Note that the $i^{th}$ dimension of $\tilde{\textbf{x}}^{j}$
is given by the formula: $\tilde{x}^{j}_{i} = h_{i,1}x^{j}_{1} + ... + h_{i,n}x^{j,n}$,
where $h_{l,u}$ stands for the $l^{th}$ element of the $u^{th}$ column of the randomized
Hadamard matrix $\textbf{HD}$.
First we use Azuma's Inequality to find an upper bound on the probability
that $|\tilde{x}^{j}_{i}| > a$, where $a=\frac{\log(n)}{\sqrt{n}}$.
By Azuma's Inequality, we have:
\begin{equation}
\mathbb{P}[|h_{i,1}x^{j}_{1} + ... + h_{i,n}x^{j,n}| \geq a] \leq 2e^{-\frac{\log^{2}(n)}{8}}.
\end{equation}
We use: $\alpha_{i} = \beta_{i} = \frac{1}{\sqrt{n}}$.
Now we take union bound over all $n$ dimensions and the proof is completed.
\end{proof}

\subsubsection{\textit{TripleSpin}-equivalent definition}
We will introduce here equivalent definition of the $\textit{TripleSpin}$-model that is more technical (thus we did not give it in the main body of the paper), yet more convenient to work with in the proofs.

Note that from the definition of the \textit{TripleSpin}-family we can conclude that each structured matrix $\textbf{G}_{struct} \in \mathbb{R}^{n \times n}$ from the \textit{TripleSpin}-family is a product of three main structured blocks, i.e.:
\begin{equation}
\textbf{G}_{struct} = \textbf{B}_{3}\textbf{B}_{2}\textbf{B}_{1},
\end{equation}

where matrices $\textbf{B}_{1},\textbf{B}_{2},\textbf{B}_{3}$ satisfy two conditions that we give below.

\begin{framed}
\textbf{Condition 1:} Matrices: $\textbf{B}_{1}$ and $\textbf{B}_{2}\textbf{B}_{1}$ are $(\delta(n),p(n))$-balanced isometries. \\
\textbf{Condition 2:} Pair of matrices $(B_{2},B_{3})$
is $(K,\Lambda_{F}, \Lambda_{2})$-random.
\end{framed}

Below we give the definition of $(K, \Lambda_{F}, \Lambda_{2})$-randomness.

\begin{definition}[$(K, \Lambda_{F}, \Lambda_{2})$-randomness]
A pair of matrices $(\textbf{Y},\textbf{Z}) \in \mathbb{R}^{n \times n} \times \mathbb{R}^{n \times n}$ is $(K, \Lambda_{F}, \Lambda_{2})$-random
if there exist: $\textbf{r} \in \mathbb{R}^{k}$, and
a set of linear isometries $\phi = \{\phi_{1},...,\phi_{n}\}$, 
where $\phi_i : \mathbb{R}^{n} \rightarrow \mathbb{R}^{k}$, such that:
\begin{itemize}
\item $\textbf{r}$ is either a $\pm 1$-vector with i.i.d. entries
      or Gaussian with identity covariance matrix,
\item for every $\textbf{x} \in \mathbb{R}^{n}$ the $j^{th}$ element $(\textbf{Zx})_{j}$ of $\textbf{Zx}$ is of the form: $\textbf{r}^{T} \cdot \phi_{j}(\textbf{x})$,     
\item there exists a set of i.i.d. sub-Gaussian random variables $\{\rho_{1},...,\rho_{n}\}$ with sub-Gaussian norm at most $K$, mean $0$, the same second moments and a $(\Lambda_{F},\Lambda_{2})$-smooth set of matrices $\{\textbf{W}^{i}\}_{i=1,...,n}$ such that for every $\textbf{x} = (x_{1},...,x_{n})^{T}$ we have: $\phi_{i}(\textbf{Y}\textbf{x}) = \textbf{W}^{i} (\rho_{1}x_{1},...,\rho_{n}x_{n})^{T}$.
\end{itemize}
\end{definition}

\subsubsection{Proof of Lemma \ref{simple_lemma}}

\begin{proof}
Let us first assume the $\textbf{G}_{circ}\textbf{D}_{2}\textbf{HD}_{1}$ setting
(analysis for Toeplitz Gaussian or Hankel Gaussian is completely analogous).
In that setting it is easy to see that one can take $\textbf{r}$ to be a Gaussian vector (this vector corresponds to the first row of $\textbf{G}_{circ}$). Furthermore linear mappings $\phi_{i}$ are defined as: $\phi_{i}((x_{0},x_{1},...,x_{n-1})^{T}) = (x_{n-i},x_{n-i+1},...,x_{i-1})^{T}$, where operations on indices are modulo $n$.
The value of $\delta(n)$ and $p(n)$ come from the fact that matrix $\textbf{HD}_{1}$ is used as a $(\delta(n),p(n))$-balanced matrix and from Remark \ref{balanceness_remark}.
In that setting sequence $(\rho_{1},...,\rho_{n})$ is discrete and corresponds to the diagonal of $\textbf{D}_{2}$.
Thus we have: $K = 1$. To calculate $\Lambda_{F}$ and $\Lambda_{2}$, note first that matrix $\textbf{W}^{1}$ is defined as $\textbf{I}$ and subsequent $\textbf{W}^{i}$s are given as circulant shifts of the previous ones (i.e. each row is a circulant shift of the previous row). That observation comes directly from the circulant structure of $\textbf{G}_{circ}$. Thus we have: $\Lambda_{F} = O(\sqrt{n})$ and $\Lambda_{2} = O(1)$. The former is true since each $\textbf{A}^{i,j}$ has $O(n)$ nonzero entries and these are all $1$s. The latter is true since each nontrivial $\textbf{A}^{i,j}$ in that setting is an isometry (this comes straightforwardly from the definition of $\{\textbf{W}^{i}\}_{i=1,...,n}$).
Finally, all other conditions regarding $\textbf{W}^{i}$-matrices are clearly satisfied (each column of each $\textbf{W}^{i}$ has unit $L_{2}$ norm and corresponding columns from different $\textbf{W}^{i}$ and $\textbf{W}^{j}$ are clearly orthogonal).

Now let us consider the setting, where the structured matrix is of the form: $\sqrt{n}\textbf{HD}_{3}\textbf{HD}_{2}\textbf{HD}_{1}$.
In that case $\textbf{r}$ corresponds to a discrete vector (namely, the diagonal of $\textbf{D}_{2}$).
Linear mappings $\phi_{i}$ are defined as:
$\phi_{i}((x_{1},...,x_{n})^{T}) = (\sqrt{n}h_{i,1}x_{1},...,\sqrt{n}h_{i,n}x_{n})^{T}$, where $(h_{i,1},...,h_{i,n})^{T}$ is the $i^{th}$ row of $\textbf{H}$.
One can also notice that the set $\{\textbf{W}^{i}\}_{i=1,...,n}$ is defined as: $w^{i}_{a,b} = \sqrt{n} h_{i,a}h_{a,b}$.
Let us first compute the Frobenius norm of the matrix $\textbf{A}^{i,j}$ defined based on the aforementioned sequence $\{\textbf{W}^{i}\}_{i=1,...,n}$.
We have:
\begin{equation}
\|\textbf{A}^{i,j}\|_{F}^{2} = \sum_{l,t \in \{1,...,n\}} 
(\sum_{k=1}^{n} w^{j}_{k,l}w^{i}_{k,t})^{2} = n^{2}
\sum_{l,t \in \{1,...,n\}} (\sum_{k=1}^{n} h_{j,k}h_{k,l}h_{i,k}h_{k,t})^{2}
\end{equation}

To compute the expression above, note first that for $r_{1} \neq r_{2}$ we have:

\begin{equation}
\theta = \sum_{k,l} h_{r_{1},k}h_{r_{1},l}h_{r_{2},k}h_{r_{2},l}
 = \sum_{k} h_{r_{1},k}h_{r_{2},k} \sum_{l} h_{r_{1},l}h_{r_{2}, l} = 0,
\end{equation}
where the last equality comes from fact that different rows of $\ {H}$ are orthogonal. From the fact that $\theta = 0$ we get:
\begin{equation}
\|\textbf{A}^{i,j}\|_{F}^{2} = n^{2} \sum_{r=1,...,n} \sum_{k,l} 
h_{i,r}^{2}h_{j,r}^{2}h_{r,k}^{2}h_{r,l}^{2} = 
n \cdot n^{2} (\frac{1}{\sqrt{n}})^{8} \cdot n^{2} = n.
\end{equation}
Thus we have: $\Lambda_{F} \leq \sqrt{n}$.

Now we compute $\|\textbf{A}^{i,j}\|_{2}$.
Notice that from the definition of $\textbf{A}^{i,j}$ we get that
\begin{equation}
\textbf{A}^{i,j} = \textbf{E}^{i,j} \textbf{F}^{i,j},
\end{equation}
where the $l^{th}$ row of $\textbf{E}^{i,j}$ is of the form
$(h_{j,1}h_{1,l},...,h_{j,n}h_{n,l})$ and the $t^{th}$ column of
$\textbf{F}^{i,j}$ is of the form $(h_{i,1}h_{1,t},...,h_{i,n}h_{n,t})^{T}$.
Thus one can easily verify that $\textbf{E}^{i,j}$ and $\textbf{H}^{i,j}$ are isometries (since $\textbf{H}$ is) thus
$\textbf{A}^{i,j}$ is also an isometry and therefore $\Lambda_{2} = 1$. As in the previous setting, remaining conditions regarding matrices $\textbf{W}^{i}$ are trivially satisfied (from the basic properties of Hadamard matrices).
That completes the proof.

\end{proof}

\subsubsection{Proof of Theorem \ref{main_struct_theorem}}

Let us briefly give an overview of the proof before presenting it in detail. Challenges regarding proving accuracy results for structured matrices come from the fact that for any given $\textbf{x} \in \mathbb{R}^{n}$
different dimensions of $\textbf{y} = \textbf{G}_{struct}\textbf{x}$ are no longer independent (as it is the case for the unstructured setting).
For matrices from the \textit{TripleSpin}-family we can however show that with high probability different elements of $\textbf{y}$ correspond to projections of a given vector $\textbf{r}$ (see Section \ref{sec:model}) into directions that are close to orthogonal. The "close-to-orthogonality" characteristic is obtained with the use of the Hanson-Wright inequality that focuses on concentration results regarding quadratic forms involving vectors of sub-Gaussian random variables. If $\textbf{r}$ is Gaussian then from the well known fact that projections of the Gaussian vector into orthogonal directions are independent we can conclude that dimensions of $\textbf{y}$ are "close to independent". If $\textbf{r}$ is a discrete vector then we need to show that for $n$ large enough it "resembles" the Gaussian vector. This is where we need to apply the aforementioned techniques regarding multivariate Berry-Esseen-type central limit theorem results.

\begin{proof}
We will use notation from Section \ref{sec:model} and previous sections of the Appendix. 
We assume that the model with structured matrices stacked vertically, each of $m$ rows is applied. Without loss of generality we can assume that we have just one block since different blocks are chosen independently.
Let $\textbf{G}_{struct}$ be a matrix from the \textit{TripleSpin}-family. Let us assume that $\textbf{G}_{struct}$ is used by a function $f$ operating in the $d$-dimensional space and let us denote by $\textbf{x}^{1}$,...,$\textbf{x}^{d}$ some fixed orthonormal basis of that space.
Our first goal is to compute: $\textbf{y}^{1} = \textbf{G}_{struct} \textbf{x}^{1},...,\textbf{y}^{d} = \textbf{G}_{struct} \textbf{x}^{d}$.
Denote by $\tilde{\textbf{x}}^{i}$ the linearly transformed version of $\textbf{x}$ after applying block $\textbf{B}_{1}$, i.e. $\tilde{\textbf{x}}^{i} = \textbf{B}_{1} \textbf{x}^{i}$.
Since $\textbf{B}_{2}$ is $(\delta(n),p(n))$-balanced), we conclude that with probability at least: $p_{balanced} \geq 1 - dp(n)$ each element of each $\tilde{\textbf{x}}^{i}$ has absolute value at most $\frac{\delta(n)}{\sqrt{n}}$. We shortly say that each $\tilde{\textbf{x}}^{i}$ is $\delta(n)$-balanced.
We call this event $\mathcal{E}_{balanced}$.

Note that by the definition of the \textit{TripleSpin}-family, each $\textbf{y}^{i}$ is of the form: 
\begin{equation}
\textbf{y}^{i} = (\textbf{r}^{T} \cdot \phi_{1}(\textbf{B}_{2}\tilde{\textbf{x}}^{i}),...,\textbf{r}^{T} \cdot \phi_{m}(\textbf{B}_{2}\tilde{\textbf{x}}^{i}))^{T}.
\end{equation}

For clarity and to reduce notation we will assume that $\textbf{r}$ is $n$-dimensional.
To obtain results for vectors $\textbf{r}$ of different dimensionality $D$ it suffices to replace in our analysis and theoretical statements $n$ by $D$.
Let us denote $\mathcal{A} = \{\phi_{1}(\textbf{B}_{2}\tilde{\textbf{x}}^{1}),...,\phi_{m}(\textbf{B}_{2}\tilde{\textbf{x}}^{1}),...,\phi_{1}(\textbf{B}_{2}\tilde{\textbf{x}}^{d}),...,\phi_{m}(\textbf{B}_{2}\tilde{\textbf{x}}^{d}))\}$.
Our goal is to show that with high probability (in respect to random choices of $\textbf{B}_{1}$ and $\textbf{B}_{2}$) for all $\textbf{v}^{i},\textbf{v}^{j} \in \mathcal{A}$, $i \neq j$ the following is true:
\begin{equation}
\label{dot_product_equation}
|(\textbf{v}^{i})^{T} \cdot \textbf{v}^{j}| \leq t
\end{equation}
for some given $0 < t \ll 1$.

Fix some $t>0$. We would like to compute
the lower bound on the corresponding probability.
Let us fix two vectors $\textbf{v}^{1}, \textbf{v}^{2} \in \mathcal{A}$ and denote them as: $\textbf{v}^{1} = \phi_{i}(\textbf{B}_{2}\textbf{x})$, $\textbf{v}^{2} = \phi_{j}(\textbf{B}_{2} \textbf{y})$ for some $\textbf{x} = (x_{1},...,x_{n})^{T}$
and $\textbf{y} = (y_{1},...,y_{n})^{T}$.
Note that we have (see denotation from Section \ref{sec:model}):
\begin{equation}
\phi_{i}(\textbf{B}_{2}\textbf{x}) = (w^{i}_{11}\rho_{1}x_{1} + ... + w^{i}_{1,n}\rho_{n}x_{n},...,w^{i}_{n,1}\rho_{1}x_{1} + ... + w^{i}_{n,n}\rho_{n}x_{n})^{T}
\end{equation}
and
\begin{equation}
\phi_{j}(\textbf{B}_{2}\textbf{y}) = (w^{j}_{11}\rho_{1}y_{1} + ... + w^{j}_{1,n}\rho_{n}y_{n},...,w^{j}_{n,1}\rho_{1}y_{1} + ... + w^{j}_{n,n}\rho_{n}y_{n})^{T}.
\end{equation}

We obtain:
\begin{equation}
(\textbf{v}^{1})^{T} \cdot \textbf{v}^{2}=
\sum_{l \in \{1,...,n\}, u\in \{1,...,n\}} \rho_{l}\rho_{u}(\sum_{k=1}^{n}x_{l}y_{u}w^{i}_{k,u}w^{j}_{k,l}).
\end{equation}

We show now that under assumptions from Theorem \ref{main_struct_theorem} the expected 
value of the above expression is $0$.
We have:
\begin{equation}
\mathbb{E}[(\textbf{v}^{1})^{T} \cdot \textbf{v}^{2}]=
\mathbb{E}[\sum_{l \in \{1,...,n\}} \rho_{l}^{2}x_{l}y_{l}
(\sum_{k=1}^{n}w^{i}_{k,l}w^{j}_{k,l})],
\end{equation}
since $\rho_{1},...,\rho_{n}$ are independent
and have expectations equal to $0$.
Now notice that if $i \neq j$ then from the assumption that
corresponding columns of matrices $\textbf{W}^{i}$ and $\textbf{W}^{j}$ are orthogonal we get that the above expectation is $0$. Now assume that $i = j$. But then $\textbf{x}$ and $\textbf{y}$ have to be different and thus they are orthogonal (since they are taken from the orthonormal system transformed by an isometry).
In that setting we get:
\begin{equation}
\mathbb{E}[(\textbf{v}^{1})^{T} \cdot \textbf{v}^{2}]=
\mathbb{E}[\sum_{l \in \{1,...,n\}} \rho_{l}^{2}x_{l}y_{l}
(\sum_{k=1}^{n}(w^{i}_{k,l})^{2})]= \tau w \sum_{l=1}^{n} x_{l}y_{l} = 0,
\end{equation}
where $\tau$ stands for the second moment of each $\rho_{i}$,
$w$ is the squared $L_{2}$-norm of each column of $\textbf{W}^{i}$
($\tau$ and $w$ are well defined due to the properties of the \textit{TripleSpin}-family). The last inequality comes from the fact that $\textbf{x}$ and $\textbf{y}$ are orthogonal.
Now if we define matrices $\textbf{A}^{i,j}$ as in the definition of the \textit{TripleSpin}-model then we see that 
\begin{equation}
(\textbf{v}^{1})^{T} \cdot \textbf{v}^{2} = 
\sum_{l,u \in \{1,...,n\}} \rho_{l}\rho_{u}T^{i,j}_{l,u},
\end{equation}
where:
$T^{i,j}_{l,u} = x_{l}y_{u}A^{i,j}_{l,u}$.

Now we will use the following inequality:
\begin{theorem}[Hanson-Wright Inequality]
Let $\textbf{X} = (X_{1},...,X_{n})^{T} \in \mathbb{R}^{n}$ be a random vector with independent components $X_{i}$ which satisfy: $\mathbb{E}[X_{i}] = 0$ and have sub-Gaussian norm at most $K$ for some given $K>0$.
Let $\textbf{A}$ be an $n \times n$ matrix. Then for every $t \geq 0$
the following is true:
\begin{equation}
\mathbb{P}[\textbf{X}^{T}\textbf{A}\textbf{X} - \mathbb{E}[\textbf{X}^{T}\textbf{AX}] > t] \leq 
2e^{-c \min(\frac{t^{2}}{K^{4}\|\textbf{A}\|^{2}_{F}},
\frac{t}{K^{2}\|\textbf{A}\|_{2}})},
\end{equation}
\end{theorem}
where $c$ is some universal positive constant.

Note that, assuming $\delta(n)$-balancedness, we have: $\|\textbf{T}^{i,j}\|_{F} \leq \frac{\delta^{2}(n)}{n} \|\textbf{A}^{i,j}\|_{F}$ and $\|\textbf{T}^{i,j}\|_{2} \leq
\frac{\delta^{2}(n)}{n}
\|\textbf{A}^{i,j}\|_{2}$.

Now we take $\textbf{X} = (\rho_{1},...,\rho_{n})^{T}$ and $\textbf{A} = \textbf{T}^{i,j}$ in the theorem above.
Applying the Hanson-Wright inequality in that setting,
taking the union bound over all pairs of different vectors $\textbf{v}^{i},\textbf{v}^{j} \in \mathcal{A}$ (this number is exactly: ${md \choose 2}$) and the event $\mathcal{E}_{balanced}$, finally taking the union bound over all $s$ functions $f_{i}$, we conclude that with probability at least: 
\begin{equation}
\label{imp_equation}
p_{good} = 1 - p(n)ds - 2{md \choose 2}s e^{-\Omega(\min(\frac{t^{2}n^{2}}{K^{4}\Lambda_{F}^{2}\delta^{4}(n)}, \frac{tn}{K^{2}\Lambda_{2} \delta^{2}(n)}))}
\end{equation}

for every $f$ any two different vectors $\textbf{v}^{i}, \textbf{v}^{j} \in \mathcal{A}$ satisfy:
$|(\textbf{v}^{i})^{T} \cdot \textbf{v}^{j}| \leq t$.

Note that from the fact that $\textbf{B}_{2}\textbf{B}_{1}$ is $(\delta(n),p(n))$-balanced and from Equation \ref{imp_equation}, we get that with probability at least:
\begin{equation}
p_{right} = 1 - 2p(n)ds - 2{md \choose 2}s e^{-\Omega(\min(\frac{t^{2}n^{2}}{K^{4}\Lambda_{F}^{2}\delta^{4}(n)}, \frac{tn}{K^{2}\Lambda_{2} \delta^{2}(n)}))}.
\end{equation}

for every $f$ any two different vectors $\textbf{v}^{i}, \textbf{v}^{j} \in \mathcal{A}$ satisfy:
$|(\textbf{v}^{i})^{T} \cdot \textbf{v}^{j}| \leq t$ and furthermore each $\textbf{v}^{i}$ is $\delta(n)$-balanced.

Assume now that this event happens.
Consider the vector 
\begin{equation}
\textbf{q}^{\prime} = ((\textbf{y}^{1})^{T},...,(\textbf{y}^{d})^{T})^{T} \in \mathbb{R}^{md}. 
\end{equation}

Note that $\textbf{q}^{\prime}$ can be equivalently represented as:
\begin{equation}
\textbf{q}^{\prime} = 
(\textbf{r}^{T} \cdot \textbf{v}^{1},...,\textbf{r}^{T} \cdot \textbf{v}^{md}),
\end{equation}
where: $\mathcal{A} = \{\textbf{v}^{1},...,\textbf{v}^{md}\}$.
From the fact that $\phi_{i}\textbf{B}_{2}$ and $\textbf{B}_{1}$ are isometries we conclude that: $\|\textbf{v}^{i}\|_{2} = 1$ for $i=1,...$. 

Now we will need the following Berry-Esseen type result for random vectors:

\begin{theorem}[Bentkus \cite{bentkus2003dependence}]
\label{clt_theorem}
Let $\textbf{X}_{1},...,\textbf{X}_{n}$ be independent vectors taken from $\mathbb{R}^{k}$ with common mean $\mathbb{E}[\textbf{X}_{i}] = 0$. Let $\textbf{S} = \textbf{X}_{1} + ... + \textbf{X}_{n}$.
Assume that the covariance operator $\textbf{C}^{2} = cov(\textbf{S})$ is invertible. Denote $\beta_{i} = 
\mathbb{E}[\|\textbf{C}^{-1}\textbf{X}_{i}\|_{2}^{3}]$
and $\beta = \beta_{1} + ... + \beta_{n}$.
Let $\mathcal{C}$ be the set of all convex subsets of $\mathbb{R}^{k}$. Denote $\Delta(\mathcal{C}) = \sup_{A \in \mathcal{C}} |\mathbb{P}[S \in A]-\mathbb{P}[Z \in A]|$,
where $Z$ is the multivariate Gaussian distribution with mean $0$ and covariance operator $\textbf{C}^{2}$. Then:
\begin{equation}
\Delta(\mathcal{C}) \leq ck^{\frac{1}{4}} \beta
\end{equation}
for some universal constant $c$.
\end{theorem}

Denote: $\textbf{X}_{i} = (r_{i}v^{1}_{i},...,r_{i}v^{k}_{i})^{T}$
for $k=md$, $\textbf{r} = (r_{1},...,r_{n})^{T}$
and $\textbf{v}^{j} = (v^{j}_{1},...,v^{j}_{n})$.
Note that $\textbf{q}^{\prime} = \textbf{X}_{1} + ... + \textbf{X}_{n}$. Clearly we have: $\mathbb{E}[\textbf{X}_{i}] = 0$.
Furthermore, given the choices of $\textbf{v}^{1},...,\textbf{v}^{k}$, random vectors $\textbf{X}_{1},..,\textbf{X}_{n}$ are independent.

Let us calculate now the covariance matrix of $\textbf{q}^{\prime}$.
We have: 
\begin{equation}
\textbf{q}^{\prime}_{i} = r_{1}v^{i}_{1} + ... + r_{n}v^{i}_{n},
\end{equation}

where: $\textbf{q}^{\prime} = (\textbf{q}^{\prime}_{1},...,\textbf{q}^{\prime}_{k})$.

Thus for $i_{1}, i_{2}$ we have:

\begin{equation}
\mathbb{E}[\textbf{q}^{\prime}_{i_{1}} \textbf{q}^{\prime}_{i_{2}}] = 
\sum_{j=1}^{n} v^{i_{1}}_{j}v^{i_{2}}_{j}\mathbb{E}[r_{j}^{2}] + 2\sum_{1 \leq j_{1} < j_{2} \leq n} v^{i_{1}}_{j_{1}}v^{i_{2}}_{j_{2}} \mathbb{E}[r_{j_{1}}r_{j_{2}}]
= (\textbf{v}^{i_{1}})^{T} \cdot \textbf{v}^{i_{2}},
\end{equation}
where the last equation comes from the fact $r_{j}$ are either Gaussian from $\mathcal{N}(0,1)$ or discrete with entries from $\{-1,+1\}$ and furthermore different $r_{j}$s are independent.

Therefore if $i_1=i_2=i$, since each $\textbf{v}^{i}$ has unit $L_{2}$-norm, we have that 
\begin{equation}
\mathbb{E}[\textbf{q}^{\prime}_{i} \textbf{q}^{\prime}_{i}] = 1, 
\end{equation}
and for $i_{1} \neq i_{2}$ we get:
\begin{equation}
|\mathbb{E}[\textbf{q}^{\prime}_{i_{1}} \textbf{q}^{\prime}_{i_{2}}]| \leq t.
\end{equation}

We conclude that the covariance matrix $\Sigma_{\textbf{q}^{\prime}}$ of the distribution $\textbf{q}^{\prime}$ is a matrix with entries $1$ on the diagonal and other entries of absolute value at most $t$.

For $t = o_{k}(1)$ small enough 
and from the $\delta(n)$-balancedness of vectors $\textbf{v}^{1},...,\textbf{v}^{k}$ 
we can conclude that:
\begin{equation}
\mathbb{E}[\|\textbf{C}^{-1}\textbf{X}_{i}\|^{3}_{2}] =
O(\mathbb{E}[\|\textbf{X}_{i}\|^{3}_{2}]) = O(\sqrt{(\frac{k}{n})^{3}}\delta^{3}(n)),
\end{equation}

Now, using Theorem \ref{clt_theorem}, we conclude that
\begin{equation}
\sup_{A \in \mathcal{C}} |\mathbb{P}[\textbf{q}^{\prime} \in A] - \mathbb{P}[Z \in A]| = O(k^{\frac{1}{4}}n \cdot \frac{k^{\frac{3}{2}}}{n^{\frac{3}{2}}} \delta^{3}(n)) = O(\frac{\delta^{3}(n)}{\sqrt{n}}k^{\frac{7}{4}}),
\end{equation}
where $Z$ is taken from the multivariate Gaussian distribution with covariance matrix $\textbf{I} + \textbf{E}$. Now if we take $\eta = \frac{\delta^{3}(n)}{\sqrt{n}}k^{\frac{7}{4}}$, $\epsilon = t = o_{md}(1)$ and take $n$ large enough, the statement of the theorem follows.
\end{proof}

\subsubsection{Proof of Theorem \ref{corollary_theorem}}

\begin{proof}
This comes directly from Theorem \ref{main_struct_theorem} and Lemma \ref{simple_lemma}.
\end{proof}

\subsubsection{Proof of Theorem \ref{hopefully_last_theorem}}

\begin{proof}
For clarity we will assume that the structured matrix consists of just one block of $m$ rows and will compare its performance with the unstructured variant of $m$ rows (the more general case when the structured matrix is obtained by stacking vertically many blocks is analogous since the blocks are chosen independently). 

Consider the two-dimensional linear space $\mathcal{H}$ spanned by $\textbf{x}$ and $\textbf{y}$.
Fix some orthonormal basis $\mathcal{B} = \{\textbf{u}^{1},\textbf{u}^{2}\}$
of $\mathcal{H}$.
Take vectors \textbf{q} and $\textbf{q}^{\prime}$. 
Note that they are $2m$-dimensional, where $m$ is the number of rows of the block used in the structured setting.
From Theorem \ref{corollary_theorem} we conclude that will probability
at least $p_{success}$, where $p_{success}$ is as in the statement of the theorem the following holds for any convex $2m$-dimensional set $A$:
\begin{equation}
|\mathbb{P}[\textbf{q}(\epsilon) \in A] - \mathbb{P}[\textbf{q}^{\prime}]| \leq \eta,
\end{equation}
where $\eta = \frac{\log^{3}(n)}{n^{\frac{2}{5}}}$.
Take two corresponding entries of vectors $\textbf{v}_{\textbf{x},\textbf{y}}^{1}$ and $\textbf{v}_{\textbf{x},\textbf{y}}^{2}$ indexed by a pair $(\textbf{e}_{i}, \textbf{e}_{j})$ for some fixed $i,j \in \{1,...,m\}$.
Call them $p^{1}$ and $p^{2}$ respectively. Our goal is to compute 
$|p^{1} - p^{2}|$.
Notice that $p^{1}$ is the probability that $h(\textbf{x}) = \textbf{e}_{i}$
and $h(\textbf{y}) = \textbf{e}_{j}$ for the unstructured setting and $p^{2}$ is that probability for the structured variant.

Let us consider now the event $E^{1} = \{h(\textbf{x}) = \textbf{e}_{i}
\land h(\textbf{y}) = \textbf{e}_{j}\}$, where the setting is unstructured.
Denote the corresponding event for the structured setting as $E^{2}$.
Denote $\textbf{q} = (q_{1},...,q_{2m})$ (for vectors of the form: $-\textbf{e}_{i}$ and $-\textbf{e}_{j}$ the analysis is exactly the same).
Assume that $\textbf{x} = \alpha_{1} \textbf{u}^{1} + \alpha_{2} \textbf{u}^{2}$ for some scalars $\alpha_{1}, \alpha_{2} > 0$.
Denote the unstructured Gaussian matrix by $\textbf{G}$.
We have:
\begin{equation}
\textbf{Gx} = \alpha_{1}\textbf{G}\textbf{u}^{1} + 
              \alpha_{2}\textbf{G}\textbf{u}^{2}
\end{equation}
Note that we have: $\textbf{G}\textbf{u}^{1} = (q_{1},...,q_{m})^{T}$
and $\textbf{G}\textbf{u}^{2} = (q_{m+1},...,q_{2m})^{T}$.
Denote by $A(\textbf{e}_{i})$ the set of all the points in $\mathbb{R}^{m}$ such that their angular distance to $\textbf{e}_{i}$ is at most the angular distance to all other $m-1$ canonical vectors. Note that this is definitely the convex set.
Now denote:
\begin{equation}
Q(\textbf{e}_{i}) = \{(q_{1},...,q_{2m})^{T} \in \mathbb{R}^{2m} : \alpha_{1} (q_{1},...,q_{m})^{T} + \alpha_{2} (q_{m+1},...,q_{2m})^{T} \in A(\textbf{e}_{i})\}.
\end{equation}
Note that since $A(\textbf{e}_{i})$ is convex, we can conclude that $Q(\textbf{e}_{i})$ is also convex. 
Note that 
\begin{equation}
\{h(\textbf{x}) = \textbf{e}_{i}\} = \{\textbf{q} \in Q(\textbf{e}_{i})\}.
\end{equation}
By repeating the analysis for the event $\{h(\textbf{y}) = \textbf{e}_{j}\}$,
we conclude that:
\begin{equation}
\{h(\textbf{x}) = \textbf{e}_{i} \land h(\textbf{y}) = \textbf{e}_{j} \} = 
\{\textbf{q} \in Y(\textbf{e}_{i},\textbf{e}_{j})\}
\end{equation}
for convex set $Y(\textbf{e}_{i},\textbf{e}_{j}) = Q(\textbf{e}_{i}) \cap Q(\textbf{e}_{j})$.
Now observe that
\begin{equation}
|p^{1}-p^{2}| = 
|\mathbb{P}[\textbf{q} \in Y(\textbf{e}_{i},\textbf{e}_{j})] - 
\mathbb{P}[\textbf{q}^{\prime} \in Y(\textbf{e}_{i},\textbf{e}_{j})]|
\end{equation}

Thus we have:
\begin{equation}
|p^{1} - p^{2}| \leq 
|\mathbb{P}[\textbf{q} \in Y(\textbf{e}_{i},\textbf{e}_{j})] - 
\mathbb{P}[\textbf{q}(\epsilon) \in Y(\textbf{e}_{i},\textbf{e}_{j})]|+
|\mathbb{P}[\textbf{q}(\epsilon) \in Y(\textbf{e}_{i},\textbf{e}_{j})] - 
\mathbb{P}[\textbf{q}^{\prime} \in Y(\textbf{e}_{i},\textbf{e}_{j})]|
\end{equation}

Therefore we have:

\begin{equation}
|p^{1} - p^{2}| \leq |\mathbb{P}[\textbf{q} \in Y(\textbf{e}_{i},\textbf{e}_{j})] - 
\mathbb{P}[\textbf{q}(\epsilon) \in Y(\textbf{e}_{i},\textbf{e}_{j})]| + \eta.
\end{equation}

Thus we just need to upper-bound:

\begin{equation}
\xi = |\mathbb{P}[\textbf{q} \in Y(\textbf{e}_{i},\textbf{e}_{j})] - 
\mathbb{P}[\textbf{q}(\epsilon) \in Y(\textbf{e}_{i},\textbf{e}_{j})]|.
\end{equation}

Denote the covariance matrix of the distribution $\textbf{q}(\epsilon)$
as $\textbf{I} + \textbf{E}$. Note that $\textbf{E}$ is equal to $0$ on the diagonal and the 
absolute value of all other off-diagonal entries is at most $\epsilon$.

Denote $k = 2m$. We have

\begin{equation}
\xi = \left|\frac{1}{(2 \pi)^{\frac{k}{2}}\sqrt{\det (I+E)}} \int_{Y(\textbf{e}_{i},\textbf{e}_{j})} 
e^{-\frac{\textbf{x}^{T}(\textbf{I}+\textbf{E})^{-1}\textbf{x}}{2}}d\textbf{x} - 
\frac{1}{(2 \pi)^{\frac{k}{2}}} \int_{Y(\textbf{e}_{i},\textbf{e}_{j})} 
e^{-\frac{\textbf{x}^{T}\textbf{x}}{2}}d\textbf{x}\right|.
\end{equation}

Expanding: $(\textbf{I}+\textbf{E})^{-1}  = \textbf{I} - \textbf{E} + \textbf{E}^{2} - ...$, noticing that $|\det(I + E) - 1| = O(\epsilon^{2m})$,
and using the above formula, we easily get:
\begin{equation}
\xi  = O(\epsilon).
\end{equation}
That completes the proof.

\end{proof}

\subsection{Newton sketches - details of experiments}
Key attributes of the Newton sketch approach are a guaranteed 
super-linear convergence with exponentially high probability for self-concordant 
functions, and a reduced computational complexity compared to the original
second-order Newton method. Such characteristics are achieved using a sketched 
version of the Hessian matrix, in place of the original one. In the proposed experiment, 
the unconstrained large scale logistic regression is considered. Mathematically 
given a set of $n$ observations $\{(a_i,y_i)\}_{i=1..n}$, with $a_i \in 
\mathbb{R}^d$ and $y_i \in \{-1,1\}$, the logistic regression problem amounts to 
finding $x \in \mathbb{R}^d$ minimizing the cost function 
$$ f(x) = \sum_{i=1}^n \log (1 + \exp(-y_i a_i^T x)) \enspace .$$
The Newton approach to solving this optimization problem entails solving at each iteration the least squares equation $\nabla^2 f(x^t) \Delta^t = - \nabla f(x^t) $, where $$\nabla^2 f(x^t) = A^T \mathrm{diag} \left( \frac{1}{1+\exp(-a_i^T x)} (1 - \frac{1}{1+\exp(-a_i^T x)}) \right) A \in \mathbb{R}^{d \times d}$$ is the Hessian matrix of $f(x^t)$, $A = [a_1^T a_2^T \cdots a_n^T] \in \mathbb{R}^{n \times d}$, $\Delta^t = x^{t+1} - x^t$ is the increment at iteration $t$ and $\nabla f(x^t) \in \mathbb{R}^d$ is the gradient of the cost function. Authors of~\cite{pilanci} propose to rather consider the sketched version of the least square equation, based on a Hessian square root of $\nabla^2 f(x^t)$, denoted $\nabla^2 f(x^t)^{1/2} = \mathrm{diag} \left( \frac{1}{1+\exp(-a_i^T x)} (1 - \frac{1}{1+\exp(-a_i^T x)}) \right)^{1/2} A \in \mathbb{R}^{n \times d}$. The least squares problem becomes at each iteration $t$
$$ \left( (S^t \nabla^2 f(x^t)^{1/2})^T S^t \nabla^2 f(x^t)^{1/2} \right) \Delta^t = - \nabla f(x^t) \enspace ,$$
where $S^t \in \mathbb{R}^{m \times n}$ is a sequence of isotropic sketch matrices.
Let's finally recall that the gradient of the cost function is
$$ \nabla f(x^t) = \sum_{i=1}^n \left( \frac{1}{1+\exp(-y_i a_i^T x)}-1 \right) y_i a_i \enspace . $$ 

\subsection{Further experiments for kernel approximation}

For kernel approximation experiments, we used also the G50C dataset which contains 550 points of dimensionality 50 ($n = 50$) drawn from multivariate Gaussians (for Gaussian kernel, bandwidth $\sigma$ is set to $17.4734$). The results are averaged over $10$ runs. 

\paragraph{Results on the G50C dataset:} The following matrices have been tested: Gaussian random matrix $\textbf{G}$,  $\textbf{G}_{Toeplitz}\textbf{D}_{2}\textbf{HD}_{1}$ $\textbf{G}_{skew-circ}\textbf{D}_{2}\textbf{HD}_{1}$, $\textbf{HD}_{g_{1},...,g_{n}}\textbf{HD}_{2}\textbf{HD}_{1}$ and $\textbf{HD}_{3}\textbf{HD}_{2}\textbf{HD}_{1}$.

In Figure \ref{kernel_approx}, for the Gaussian kernel, all curves are almost identical. For both kernels, all \textit{TripleSpin}-matrices perform similarly to a random Gaussian matrix. 
$\textbf{HD}_{3}\textbf{HD}_{2}\textbf{HD}_{1}$ performs better than a random matrix and other \textit{TripleSpin}-matrices for a wide range of sizes of random feature maps. 

\begin{figure}[!t]
\centering
\includegraphics[scale = 0.19]{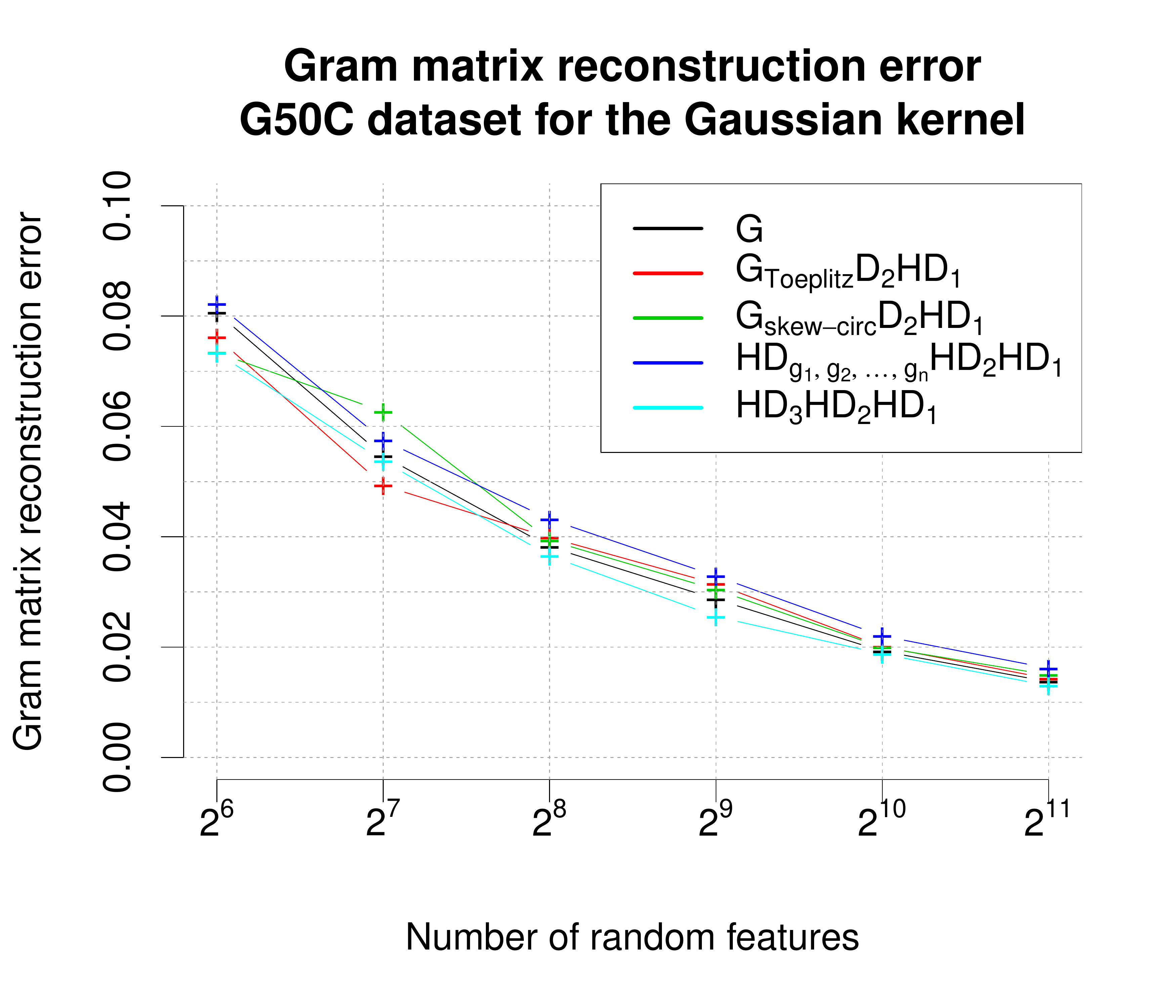}
\includegraphics[scale = 0.19]{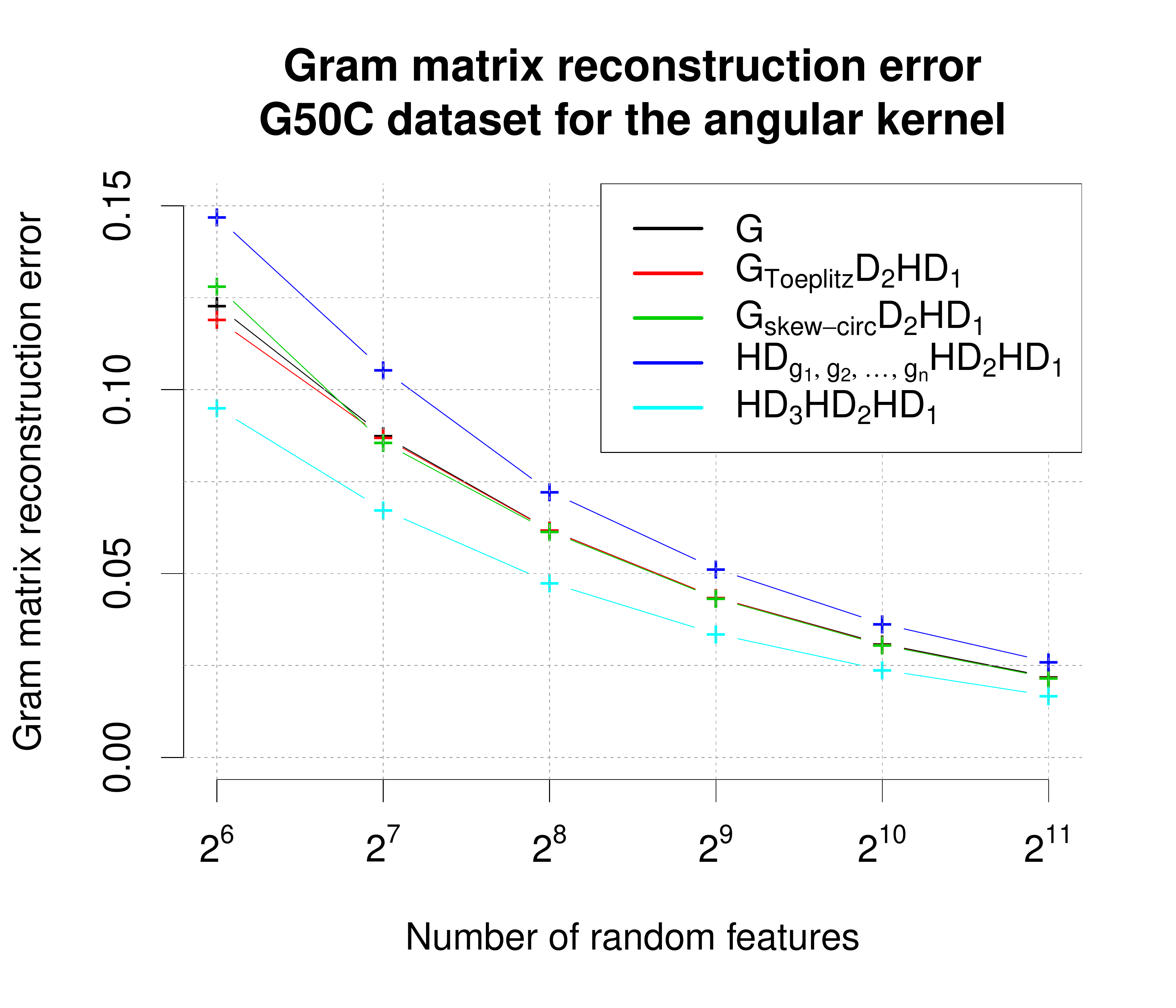}
\caption{Accuracy of random feature map kernel approximation}
\label{kernel_approx}
\end{figure}

\end{document}